\documentclass{sysu_preprint}
\DeclareTextFontCommand{\textbf}{\bfseries}
\usepackage{fix-cm}
\usepackage[T1]{fontenc}
\usepackage[utf8]{inputenc}
\usepackage{amsmath,amssymb,amsthm,mathtools}
\usepackage{graphicx,booktabs,multirow,tabularx,array,longtable}
\usepackage{colortbl}
\definecolor{baselinebest}{RGB}{225,241,225}
\definecolor{baselinesecond}{gray}{0.93}
\usepackage{algorithm,algpseudocode}
\usepackage{float}
\usepackage{hyperref,url}
\hypersetup{colorlinks=true,linkcolor=blue,citecolor=blue,urlcolor=blue,pdftitle={MA-FPPO: Multi-Agent Flow-Pretrained Policy Optimization},pdfauthor={Guowei Zou, Haonan Chen, Haitao Wang, Beiwen Zhang, Na Yan, Hejun Wu}}
\newcommand{\method}{MA-FPPO}
\newcommand{\E}{\mathbb{E}}
\newcommand{\KL}{D_{\mathrm{KL}}}
\newcommand{\Normal}{\mathcal{N}}
\newcommand{\D}{\mathcal{D}}
\newcommand{\sg}{\operatorname{sg}}
\newcommand{\clip}{\operatorname{clip}}
\newtheorem{proposition}{Proposition}
\makeatletter
\let\tableinput\@@input
\makeatother
\title{MA-FPPO: Multi-Agent Flow-Pretrained Policy Optimization}
\renewcommand{\authorlist}{\authorfont{\sffamily Guowei Zou\textsuperscript{1,2}, Haonan Chen\textsuperscript{2}, Haitao Wang\textsuperscript{1}, Beiwen Zhang\textsuperscript{1}, Na Yan\textsuperscript{1}, Hejun Wu\textsuperscript{1}}}
\renewcommand{\affiliationlist}{\affiliationfont{\sffamily\textsuperscript{1}Sun Yat-sen University, \textsuperscript{2}National University of Singapore}}
\abstract{Multi-agent flow policies learn cooperative behavior from fixed offline datasets, but often struggle to complete tasks in situations not covered by the offline data. In these situations, agents must both adapt to changes in the environment and coordinate with one another, yet action patterns learned offline are often insufficient for effective adaptation and coordination. To address this problem, we propose Multi-Agent Flow-Pretrained Policy Optimization (MA-FPPO), which uses online fine-tuning to improve the cooperative behavior of models pretrained with flow matching through new interactions with the environment. Building on the behavior learned during pretraining, we construct policies with explicit action likelihoods for discrete and continuous action spaces. We then update the pretrained model using shared team advantages to further improve coordination based on team performance. Our method achieves, on average, relative gains of 52.8\% over the strongest listed offline baselines across 30 settings and 29.8\% over purely online learning across 38 comparisons with matched online budgets and evaluation protocols.}
\newlength{\fullfigurewidth}
\newcommand{\teaserfigure}{\begin{center}
\centering
\includegraphics[width=.98\linewidth]{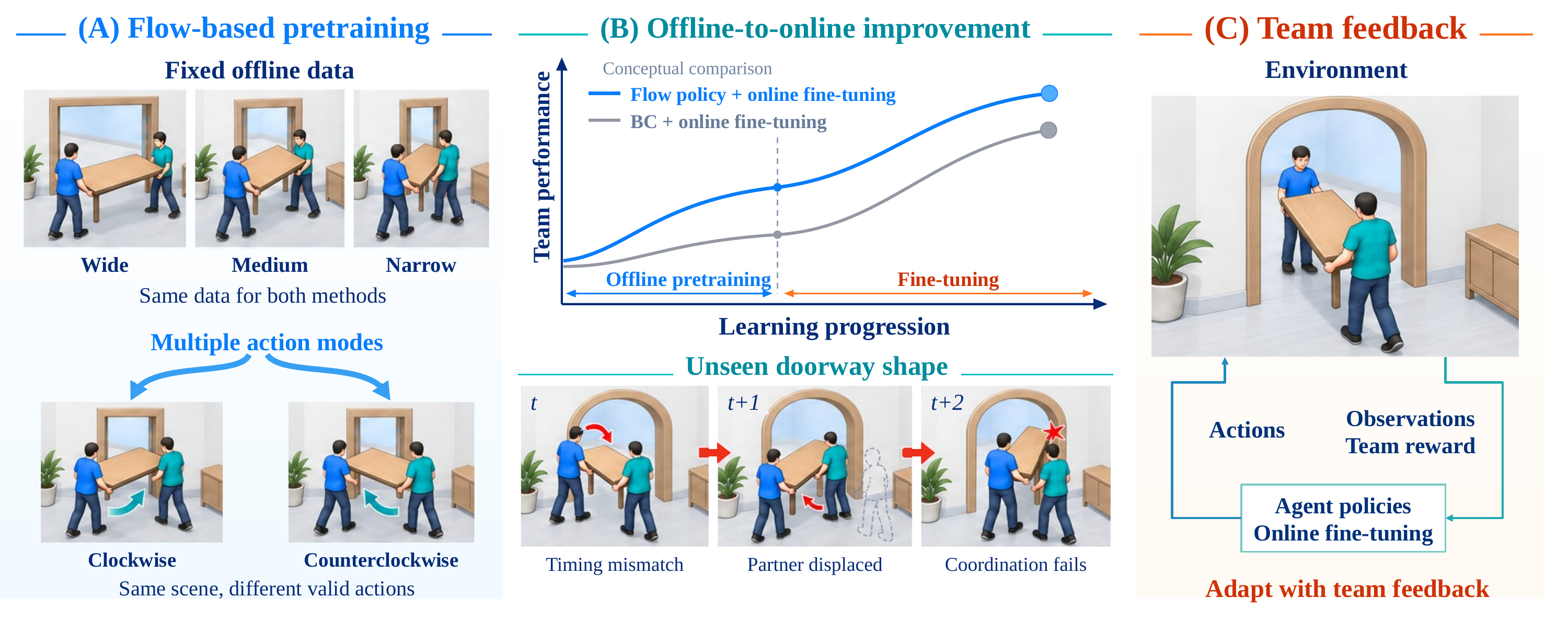}
\captionsetup{hypcap=false}
\captionof{figure}{\textbf{From learned behavior to team policy improvement.} \textbf{(A)} Flow pretraining models alternative coordinated actions in a fixed dataset, illustrated by two agents carrying a table through doorways. \textbf{(B)} A changed doorway illustrates how locally plausible actions can fail to coordinate. The curves schematically compare flow and behavioral cloning pretraining, each followed by online fine-tuning. \textbf{(C)} Team rewards guide online fine-tuning to improve coordination.}
\label{fig:motivation}
\end{center}}
\definecolor{coverpurple}{HTML}{49308C}
\definecolor{covergray}{HTML}{F1F4F8}
\renewcommand{\maketitle}{%
  \thispagestyle{plain}%
  \begin{tcolorbox}[enhanced,colback=covergray,frame hidden,arc=3mm,
    left=6mm,right=6mm,top=4mm,bottom=4mm,boxsep=0pt,before skip=0pt,after skip=5mm]
    {\raggedright\sffamily\bfseries\color{coverpurple}\fontsize{17}{20}\selectfont
      \titlelist\par}
    \vspace{2mm}
    {\raggedright\renewcommand{\authorfont}{\fontsize{10}{12}\selectfont}
      \bfseries\authorlist\par}
    \vspace{2mm}
    {\raggedright\sffamily\bfseries\color{coverpurple}\fontsize{11}{13}\selectfont \textsuperscript{1}Sun Yat-sen University\par\textsuperscript{2}National University of Singapore\par}
    \vspace{3mm}
    {\color{coverpurple!30}\hrule height 0.3pt}
    \vspace{1mm}
    \teaserfigure
    \vspace{1mm}
    {\color{coverpurple!30}\hrule height 0.3pt}
    \vspace{2mm}
    {\small\hypersetup{urlcolor=coverpurple}
      {\sffamily\bfseries Project Page:} \url{https://ma-fppo.github.io/}\par
      {\sffamily\bfseries Code:} \url{https://github.com/ma-fppo/MA-FPPO}\par
      {\sffamily\bfseries Models:} \url{https://huggingface.co/ma-fppo/MA-FPPO}\par
      {\sffamily\bfseries Datasets:} \url{https://huggingface.co/datasets/Guowei-Zou/CoFlow-datasets}\par}
  \end{tcolorbox}
  {\centering\sffamily\bfseries Abstract\par}
  \vspace{2mm}
  {\setlength{\parindent}{0pt}\abstractlist\par}
  \par\vspace{3mm}
}

\begin{document}
\raggedbottom
\addtocontents{toc}{\protect\setcounter{tocdepth}{-1}}
\maketitle
\section{Introduction}
Cooperative multi-agent control requires locally chosen actions to achieve a shared team objective~\citep{mpe,smac,facmac}. Offline multi-agent flow policies learn coordinated behavior from fixed datasets by modeling conditional action distributions with flow matching~\citep{fm,macflow}. Student distillation enables execution with one forward pass per decision~\citep{macflow}. Figure~\ref{fig:motivation}(A) illustrates alternative coordinated actions that flow pretraining can model. However, the behavior learned by these policies remains constrained by the quality and coverage of the offline data~\citep{ogmarl,kumar2020conservative}.

Offline multi-agent flow policies frequently struggle to complete tasks in situations not covered by the offline data. Learning action patterns from recorded trajectories does not ensure effective coordination when agents encounter unfamiliar states~\citep{ogmarl,kumar2020conservative}. In Figure~\ref{fig:motivation}(B), for example, agents carrying a table must adjust movements to pass through a changed doorway. Repeating familiar actions can prevent task completion because one agent's movement changes the response required from the teammate~\citep{mpe,macflow}. Offline value guidance can favor promising actions during flow pretraining, but the value estimates also depend on the existing dataset~\citep{macflow,kumar2020conservative}. Further environment interaction offers a way to improve the pretrained behavior beyond these data limitations.

The central question is how a pretrained multi-agent flow policy can improve task performance through further interaction. Existing multi-agent generative approaches emphasize offline learning and execution~\citep{madiff,li2025dof,macflow}. Online fine-tuning of the resulting pretrained students requires a policy representation that supports updates using action likelihoods~\citep{ppo,dppo,reinflow} while retaining the learned behavior as the starting point.

To enable online fine-tuning, we propose Multi-Agent Flow-Pretrained Policy Optimization (\method{}), a framework that learns cooperative behavior offline and improves the pretrained policy through online interaction. Figure~\ref{fig:motivation}(C) illustrates how team rewards guide online fine-tuning.

To build on pretrained behavior during online fine-tuning, we construct policies with explicit action likelihoods. For continuous actions, we use the pretrained student output at a fixed latent as the Gaussian mean and learn the standard deviation to control exploration, preserving the initial deterministic action. For discrete actions, student outputs define masked categorical policies. In both cases, explicit action likelihoods enable student updates~\citep{ppo} guided by shared team advantages from a centralized critic~\citep{mappo,gae}. Optional reference Kullback--Leibler (KL) regularization constrains deviation from the initial policy.

Our contributions are threefold:
\begin{itemize}
\item We propose \method{}, which enables pretrained multi-agent flow policies to improve coordination through online interaction, addressing the limited adaptability of behavior learned from fixed offline data.
\item We construct online policies for discrete and continuous action spaces that preserve the pretrained student's initial deterministic behavior and provide explicit action likelihoods for updating the pretrained model using shared team advantages.
\item We evaluate the framework across cooperative tasks and data qualities, with comparisons examining flow pretraining and online policy construction. Results show average relative gains of 52.8\% over the strongest listed offline baselines and 29.8\% over online reinforcement learning from scratch.
\end{itemize}

\section{Related Work}
\subsection{Generative Models for Multi-Agent Learning}
Generative policies model multimodal actions and trajectories for planning and control~\citep{janner2022planning,chi2023diffusion,wang2023diffusion}. Multi-agent diffusion methods, including multi-agent diffusion (MADiff), diffusion factorization (DoF), and multi-agent offline coordination via diffusion-based trajectory stitching (MADiTS), capture coordination through joint modeling, factorization, or trajectory augmentation~\citep{madiff,li2025dof,yuan2025madits}. Consistency models, shortcut models, and average velocity fields reduce sampling cost~\citep{song2023consistency,frans2025shortcut,geng2025avgvelocity}. Multi-agent Coordination via Flow Matching (MAC-Flow), Offline Multi-Agent Mean-Flow Policy (OM2P), and CoFlow reduce multi-agent sampling steps through distillation or direct generation~\citep{macflow,om2p,coflow}. We study online fine-tuning of the pretrained student.

\subsection{Online Fine-Tuning of Generative Policies}
Generative learning improves offline policies through guidance from Q values, distillation, and objectives weighted by energy~\citep{wang2023diffusion,park2025flowql,zhang2025efm}. Online methods address exploration and policy updates through likelihoods of diffusion transitions, stochastic flow paths, or MeanFlow fine-tuning, as in Diffusion Policy Policy Optimization (DPPO), ReinFlow, and One-Step Generative Policy Optimization (OGPO)~\citep{dppo,reinflow,dmpo}. $\pi_{\mathrm{RL}}$ extends fine-tuning of flow policies to vision-language-action models~\citep{pirl}. MA-FPPO constructs categorical or Gaussian online policies from the pretrained student at a fixed latent, preserving the initial deterministic action and providing explicit action likelihoods. We compare pretraining choices under the same online algorithm and online policy choices after flow pretraining.

\subsection{Multi-Agent Reinforcement Learning}
Cooperative multi-agent reinforcement learning (MARL) learns decentralized policies through value decomposition~\citep{vdn,rashid2018qmix,son2019qtran} or centralized actor--critic training, including Multi-Agent Deep Deterministic Policy Gradient (MADDPG), Counterfactual Multi-Agent Policy Gradients (COMA), Factored Multi-Agent Centralised Policy Gradients (FACMAC), and Multi-Agent Proximal Policy Optimization (MAPPO)~\citep{mpe,coma,facmac,mappo}. Trust regions and proximal objectives support conservative and sequential updates~\citep{trpo,ppo,kuba2022heterogeneous}. Offline methods include Implicit Constraint Q-learning (ICQ), Offline Multi-Agent Reinforcement Learning with Actor Rectification (OMAR), and Offline Multi-Agent Reinforcement Learning with Implicit Global-to-Local Value Regularization (OMIGA)~\citep{yang2021believe,omar,omiga}, while model-based coordination, low-rank interactions, and in-sample optimization exploit multi-agent structure~\citep{barde2024offline,zhan2025exploiting,liu2025inspo}. Off-the-Grid Multi-Agent Reinforcement Learning (OG-MARL) and other benchmark resources support standardized evaluation~\citep{ogmarl,formanek2024dispelling,smacv2}. We study how flow pretraining and online policy construction affect multi-agent policy optimization~\citep{mappo}.

Our comparison separates the value of flow pretraining from the choice of online policy. Flow and BC initialization are compared under Gaussian MAPPO, while Gaussian and generative MAPPO are compared after flow pretraining. The Gaussian policy retains pretrained parameters while changing online action sampling.

\begin{figure}[!t]
\centering
\includegraphics[width=\linewidth]{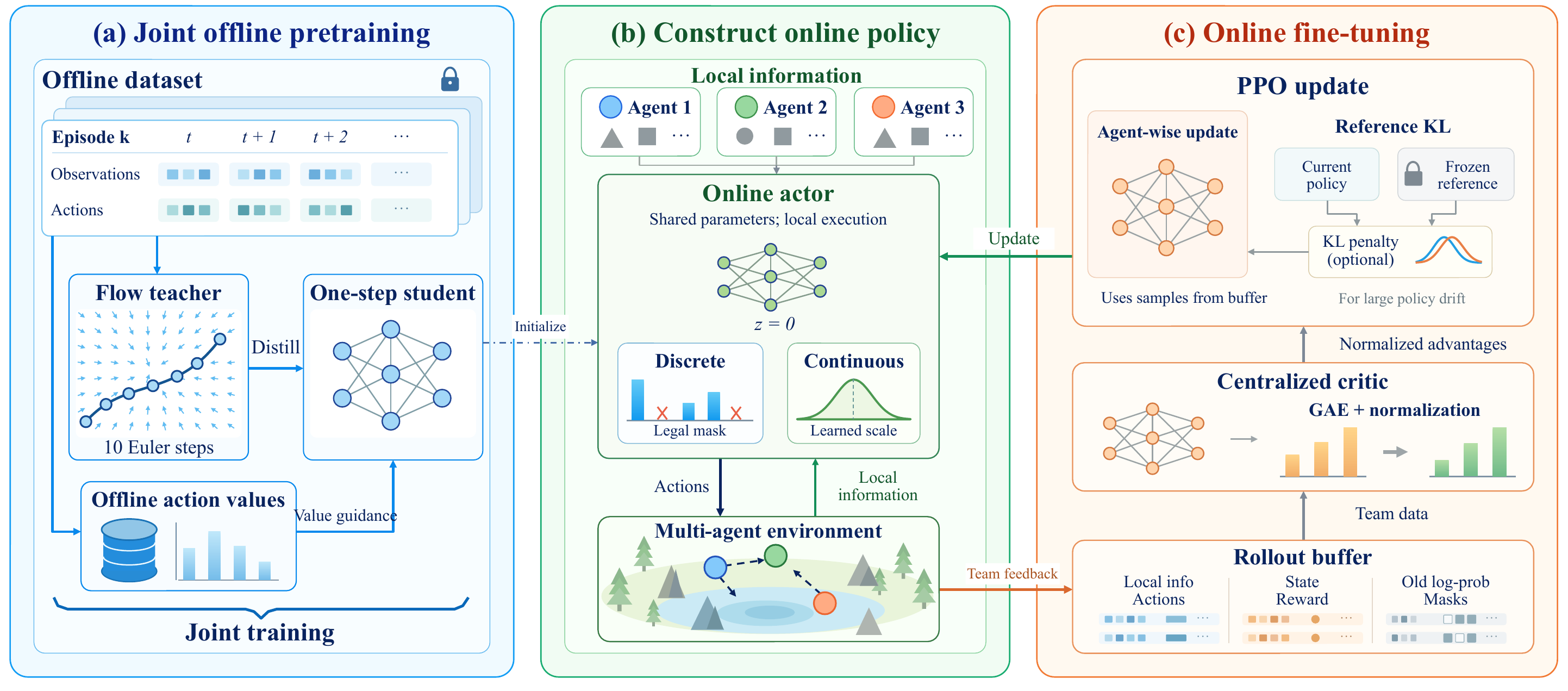}
\caption{\textbf{Overview of \method{}.} (a) Joint offline pretraining on fixed offline data. (b) Online policy construction from the pretrained student. (c) Online fine-tuning with a centralized critic. The reference KL branch depicts the optional reference KL regularization used in the configuration in the main tables. Teacher integration and offline action values are used only during pretraining. Continuous PPO evaluates raw action samples before clipping.}
\label{fig:framework}
\end{figure}

\section{Multi-Agent Flow-Pretrained Policy Optimization}

Figure~\ref{fig:framework} summarizes how flow pretraining initializes explicit online policies, which are updated using shared team advantages.

\subsection{Problem Setting}
We consider a cooperative partially observed system with $N$ controlled agents. At environment step $t$, agent $i$ observes $o_t^i$, uses local information $h_t^i$, and chooses $a_t^i$, where $h_t^i$ contains the observation or an encoding of the observation history, together with the agent identity. The joint action $\mathbf a_t$ induces a transition and team reward $r_t$. Centralized information $s_t$ is available during training. Execution uses only $h_t^i$. Appendix~\ref{app:notation} collects the notation and index conventions. A shared actor parameterization defines
\begin{equation}
 \Pi_\theta(\mathbf a_t\mid\mathbf h_t)
 =\prod_{i=1}^{N}\pi_\theta(a_t^i\mid h_t^i),
 \qquad
 J(\theta)=\E_{\Pi_\theta}\!\left[\sum_{t\geq0}\gamma^t r_t\right],
 \label{eq:joint}
\end{equation}
where $\pi_\theta$ is the executed action policy with shared parameters $\theta$, $\mathbf h_t=(h_t^1,\ldots,h_t^N)$, and $\gamma\in(0,1)$ is the discount factor. In continuous control, $\pi_\theta$ is induced by clipping the raw Gaussian sample defined below. Actions are conditionally independent given local inputs, while shared parameters, dynamics, and centralized training couple learning.

An offline dataset $\D$ supplies trajectories for pretraining. We count \emph{offline optimizer updates} separately from \emph{online environment transitions}: one transition advances one parallel environment with the joint action. The experimental setup specifies budgets.

\subsection{Offline Pretraining: Flow Learning and Student Distillation}
We jointly train a flow teacher, a student that predicts actions in one forward pass, and offline action values following the flow pretraining procedure~\citep{macflow}. Let $x_1^i\in\mathbb R^{d_a}$ denote a dataset action, encoded as a one-hot vector for discrete tasks. Sample $z_i\sim\Normal(0,I)$ and $\tau\sim\mathcal U[0,1]$, and set $x_\tau^i=(1-\tau)z_i+\tau x_1^i$. The action dimension is $d_a$, and $\tau$ denotes flow time. The teacher velocity field $v_\phi$ minimizes~\citep{fm}:
\begin{equation}
 \mathcal L_{\mathrm{FM}}
 =\E_{\D,z,\tau,i}\left[\frac{1}{d_a}\left\|v_\phi(h_i,x_\tau^i,\tau)-(x_1^i-z_i)\right\|_2^2\right],
 \label{eq:fm}
\end{equation}
where the expectation averages dataset samples, latents, flow times, and agents. Offline notation omits the environment index $t$.
The student $g_\theta(h_i,z_i)$ distills the teacher's output after 10 Euler steps using the same local input and latent. Distillation uses cross-entropy for discrete actions and squared error for continuous actions. Value guidance favors actions with higher estimated return. The joint objective is
\begin{equation}
 \mathcal L_{\mathrm{off}}
 =\mathcal L_Q+\mathcal L_{\mathrm{FM}}
 +\alpha\mathcal L_{\mathrm{distill}}+\mathcal L_{\mathrm{guide}},
 \label{eq:offline}
\end{equation}
where $\mathcal L_Q$ fits action values, $\mathcal L_{\mathrm{distill}}$ matches detached teacher outputs, $\alpha$ weights distillation, and $\mathcal L_{\mathrm{guide}}$ updates the student with value guidance while holding Q weights fixed. All terms are optimized jointly. The resulting student initializes the online actor. Appendix~\ref{app:offline} gives the teacher integration, discrete and continuous loss definitions, Q targets, architectures, and gradient routing. Algorithm~\ref{alg:offline} gives the full procedure. The comparison with BC in Section~\ref{sec:zero-step-ablation} evaluates the complete pretraining procedure.

\subsection{Online Fine-Tuning}
\subsubsection{From the Pretrained Student to an Online Policy}
We retain the pretrained student weights and fix the latent to zero, preserving the initial deterministic deployment action. Online exploration then comes from categorical or Gaussian sampling, with explicit action likelihoods for proximal policy optimization (PPO). Online updates change the student parameters, while the teacher and offline Q functions are unused. The full generative action distribution is not preserved.

\paragraph{Discrete actions.}
Let $\ell_\theta(h_i)=g_\theta(h_i,0)$ denote student logits and $\mathcal A_i(h_i)$ the nonempty set of available actions. We define
\begin{equation}
 \pi_\theta(a_i\mid h_i)
 =\frac{\exp(\ell_{\theta,a_i}(h_i)/T)}
 {\sum_{b\in\mathcal A_i(h_i)}\exp(\ell_{\theta,b}(h_i)/T)},
 \quad a_i\in\mathcal A_i(h_i).
 \label{eq:categorical}
\end{equation}
Unavailable actions are masked during collection and replay. Deployment selects the legal action with the highest logit. The temperature $T>0$ and conventions for activity masks are specified in Appendix~\ref{app:online-details}.

\paragraph{Continuous actions.}
The student output defines the Gaussian mean, with a learned standard deviation per action coordinate shared across states and agents:
\begin{equation}
 u_i\sim q_\theta(\cdot\mid h_i)
 =\Normal\!\left(g_\theta(h_i,0),\operatorname{diag}(\sigma_\theta^2)\right),
 \qquad a_i=\clip(u_i,-1,1).
 \label{eq:gaussian}
\end{equation}
PPO evaluates the raw density $q_\theta$ of stored samples $u_i$ before clipping. The vector $\sigma_\theta$ controls exploration, while team advantages update the student mean used after clipping for deployment. Separating behavior from exploration is consistent with regulated motor variability in animals~\citep{dhawale2019adaptive}, as discussed in Appendix~\ref{app:exploration-rationale}. Under the simplified model and update assumptions in Appendix~\ref{app:latent-return}, fixing $z=0$ reduces optimization error and improves limiting expected return over latent sampling.

\begin{table}[!t]
\centering
\caption{\textbf{Published baselines and MA-FPPO on continuous control.} Panels report MPE and the distinct OG-MARL and OMIGA MuJoCo collections, with separate baseline headers. Sources and dataset details are in Appendix~\ref{app:experimental-protocol}.}
\label{tab:continuous-main}
\label{tab:published_continuous}
\label{tab:mujoco_combined}
\label{tab:mujoco_ogmarl}
\label{tab:mujoco_omiga}
\begingroup\scriptsize
\renewcommand{\tiny}{\fontsize{5}{6}\selectfont}
\setlength{\tabcolsep}{2pt}
\renewcommand{\arraystretch}{1.10}
\resizebox{\linewidth}{!}{%
\begin{tabular*}{\dimexpr1.16\linewidth+12pt\relax}{@{\extracolsep{\fill}}lllccccccc|@{\extracolsep{0pt}\hspace{0.35pt}}>{\centering\arraybackslash}p{64pt}@{}}
\toprule
\multicolumn{2}{c}{\multirow{2}{*}{\textbf{Scenarios}}} & \multirow{2}{*}{\textbf{Dataset}} & \multicolumn{2}{c}{Offline single-agent extensions} & \multicolumn{5}{c|}{Offline MARL} & \textbf{Off. + On.} \\
\cmidrule(lr){4-5}\cmidrule(lr){6-10}\cmidrule(lr){11-11}
 & & & MATD3BC & MACQL & ICQ & OMAR & DOM2 & MADiff & MAC-Flow & \textbf{MA-FPPO} \\
\midrule
\multirow{10}{*}{\rotatebox[origin=c]{90}{\textbf{MPE}}} & \multirow{4}{*}{\textbf{Spread}} & Expert & $108.3$\,{\tiny$\pm 3.3$} & $98.2$\,{\tiny$\pm 5.2$} & $114.9$\,{\tiny$\pm 2.6$} & $104.0$\,{\tiny$\pm 3.4$} & \cellcolor{baselinebest}$\mathbf{131.32}$\,{\tiny$\pm 4.82$} & $95.0$\,{\tiny$\pm 5.3$} & $101.7$\,{\tiny$\pm 10.9$} & \cellcolor{baselinesecond}$\underline{116.87}^{\dagger}$\,{\tiny$\pm 0.37$} \\
 & & Medium & $29.3$\,{\tiny$\pm 4.8$} & $34.1$\,{\tiny$\pm 7.2$} & $47.9$\,{\tiny$\pm 18.9$} & $29.3$\,{\tiny$\pm 5.5$} & $55.77$\,{\tiny$\pm 7.06$} & $64.9$\,{\tiny$\pm 7.7$} & \cellcolor{baselinesecond}$\underline{80.1}$\,{\tiny$\pm 20.6$} & \cellcolor{baselinebest}$\mathbf{87.45}^{\dagger}$\,{\tiny$\pm 1.25$} \\
 & & M-R & $15.4$\,{\tiny$\pm 5.6$} & $20.0$\,{\tiny$\pm 8.4$} & $37.9$\,{\tiny$\pm 12.3$} & $13.6$\,{\tiny$\pm 5.7$} & \cellcolor{baselinesecond}$\underline{46.02}$\,{\tiny$\pm 10.81$} & $30.3$\,{\tiny$\pm 2.5$} & \cellcolor{baselinebest}$\mathbf{50.4}$\,{\tiny$\pm 33.2$} & $39.50^{\dagger}$\,{\tiny$\pm 8.70$} \\
 & & Random & $9.8$\,{\tiny$\pm 4.9$} & $24.0$\,{\tiny$\pm 9.8$} & $34.4$\,{\tiny$\pm 5.3$} & $6.3$\,{\tiny$\pm 3.5$} & \cellcolor{baselinesecond}$\underline{49.86}$\,{\tiny$\pm 7.28$} & $6.9$\,{\tiny$\pm 3.1$} & $31.1$\,{\tiny$\pm 6.8$} & \cellcolor{baselinebest}$\mathbf{128.56}^{\dagger}$\,{\tiny$\pm 1.52$} \\
\cmidrule(lr){2-11}
 & \multirow{3}{*}{\textbf{Tag}} & Expert & $115.2$\,{\tiny$\pm 12.5$}$^{\mathrm{m}}$ & $93.9$\,{\tiny$\pm 14.0$}$^{\mathrm{m}}$ & $113.0$\,{\tiny$\pm 14.4$}$^{\mathrm{m}}$ & $116.2$\,{\tiny$\pm 19.8$}$^{\mathrm{m}}$ & \cellcolor{baselinesecond}$\underline{138.75}$\,{\tiny$\pm 12.02$} & $120.9$\,{\tiny$\pm 14.6$}$^{\mathrm{m}}$ & $119.97^{\mathrm{r}}$\,{\tiny$\pm 2.26$} & \cellcolor{baselinebest}$\mathbf{205.02}^{\dagger}$\,{\tiny$\pm 2.89$} \\
 & & Medium & $65.1$\,{\tiny$\pm 29.5$}$^{\mathrm{m}}$ & $61.7$\,{\tiny$\pm 23.1$}$^{\mathrm{m}}$ & $63.3$\,{\tiny$\pm 20.0$}$^{\mathrm{m}}$ & $66.7$\,{\tiny$\pm 23.2$}$^{\mathrm{m}}$ & \cellcolor{baselinesecond}$\underline{84.29}$\,{\tiny$\pm 25.36$} & $77.2$\,{\tiny$\pm 10.4$}$^{\mathrm{m}}$ & $80.47^{\mathrm{r}}$\,{\tiny$\pm 1.77$} & \cellcolor{baselinebest}$\mathbf{159.91}^{\dagger}$\,{\tiny$\pm 2.88$} \\
 & & Random & $5.7$\,{\tiny$\pm 3.5$}$^{\mathrm{m}}$ & $5.0$\,{\tiny$\pm 8.2$}$^{\mathrm{m}}$ & $2.2$\,{\tiny$\pm 2.6$}$^{\mathrm{m}}$ & $11.1$\,{\tiny$\pm 2.8$}$^{\mathrm{m}}$ & \cellcolor{baselinesecond}$\underline{112.18}$\,{\tiny$\pm 30.21$} & $3.2$\,{\tiny$\pm 4.0$}$^{\mathrm{m}}$ & $72.68^{\mathrm{r}}$\,{\tiny$\pm 3.20$} & \cellcolor{baselinebest}$\mathbf{228.15}^{\dagger}$\,{\tiny$\pm 4.04$} \\
\cmidrule(lr){2-11}
 & \multirow{3}{*}{\textbf{World}} & Expert & $110.3$\,{\tiny$\pm 21.3$}$^{\mathrm{m}}$ & $71.9$\,{\tiny$\pm 28.1$}$^{\mathrm{m}}$ & $109.5$\,{\tiny$\pm 22.8$}$^{\mathrm{m}}$ & $110.4$\,{\tiny$\pm 25.7$}$^{\mathrm{m}}$ & \cellcolor{baselinesecond}$\underline{123.17}$\,{\tiny$\pm 19.81$} & $122.6$\,{\tiny$\pm 14.4$}$^{\mathrm{m}}$ & $122.57^{\mathrm{r}}$\,{\tiny$\pm 2.91$} & \cellcolor{baselinebest}$\mathbf{244.50}^{\dagger}$\,{\tiny$\pm 5.16$} \\
 & & Medium & $73.4$\,{\tiny$\pm 9.3$}$^{\mathrm{m}}$ & $58.6$\,{\tiny$\pm 11.2$}$^{\mathrm{m}}$ & $71.9$\,{\tiny$\pm 20.0$}$^{\mathrm{m}}$ & $74.6$\,{\tiny$\pm 11.5$}$^{\mathrm{m}}$ & $105.79$\,{\tiny$\pm 27.11$} & \cellcolor{baselinesecond}$\underline{123.5}$\,{\tiny$\pm 4.5$}$^{\mathrm{m}}$ & $89.82^{\mathrm{r}}$\,{\tiny$\pm 2.25$} & \cellcolor{baselinebest}$\mathbf{172.64}^{\dagger}$\,{\tiny$\pm 4.85$} \\
 & & Random & $2.8$\,{\tiny$\pm 5.5$}$^{\mathrm{m}}$ & $0.6$\,{\tiny$\pm 2.0$}$^{\mathrm{m}}$ & $1.0$\,{\tiny$\pm 3.2$}$^{\mathrm{m}}$ & $5.9$\,{\tiny$\pm 5.2$}$^{\mathrm{m}}$ & $54.23$\,{\tiny$\pm 16.57$} & $2.0$\,{\tiny$\pm 3.0$}$^{\mathrm{m}}$ & \cellcolor{baselinesecond}$\underline{94.43}^{\mathrm{r}}$\,{\tiny$\pm 3.72$} & \cellcolor{baselinebest}$\mathbf{273.53}^{\dagger}$\,{\tiny$\pm 3.80$} \\
\cmidrule(lr){2-11}
\multicolumn{3}{l}{\textbf{Spread average}} & $40.7$ & $44.1$ & $58.8$ & $38.3$ & \cellcolor{baselinesecond}$\underline{70.74}$ & $49.2$ & $65.8$ & \cellcolor{baselinebest}$\mathbf{93.09}^{\dagger}$\,{\tiny$\pm 2.88$} \\
\bottomrule
\end{tabular*}}
\endgroup
\par\vspace{2pt}
\begingroup
\scriptsize
\renewcommand{\tiny}{\fontsize{5}{6}\selectfont}
\setlength{\tabcolsep}{0.35pt}
\renewcommand{\arraystretch}{1.10}
\newlength{\mujocotablewidth}
\setlength{\mujocotablewidth}{\dimexpr1.16\linewidth+12pt\relax}
\newlength{\mujocobaselinewidth}
\setlength{\mujocobaselinewidth}{\dimexpr(\mujocotablewidth-13pt-51pt-31pt-64pt-18\tabcolsep-\arrayrulewidth)/6\relax}
\resizebox{\linewidth}{!}{%
\begin{tabular}{@{}p{13pt}p{51pt}p{31pt}*{6}{>{\centering\arraybackslash}p{\mujocobaselinewidth}}|>{\centering\arraybackslash}p{64pt}@{}}
\toprule
\multicolumn{2}{c}{\multirow{2}{*}{\textbf{Scenarios}}} & \multirow{2}{*}{\textbf{Dataset}} & \multicolumn{2}{c}{Offline single-agent extensions} & \multicolumn{4}{c|}{Offline MARL} & \mbox{\textbf{Off. + On.}} \\
\cmidrule(lr){4-5}\cmidrule(lr){6-9}\cmidrule(lr){10-10}
 & & & BC & MA-TD3+BC & OMAR & MADiff-D & MADiff-C & MAC-Flow$^{\mathrm{r}}$ & \textbf{MA-FPPO} \\
\midrule
\multirow{9}{*}{\rotatebox[origin=c]{90}{\textbf{OG-MARL}}} & \multirow{3}{*}{\textbf{2ant}} & Good & $2697$\,{\tiny$\pm 267$} & $2922$\,{\tiny$\pm 194$} & $464$\,{\tiny$\pm 469$} & $2946$\,{\tiny$\pm 77$} & \cellcolor{baselinesecond}$\underline{3069}$\,{\tiny$\pm 60$} & $1628.46^{\mathrm{r}}$\,{\tiny$\pm 1.72$} & \cellcolor{baselinebest}$\mathbf{3436.15}$\,{\tiny$\pm 18.40$} \\
 & & Medium & $1145$\,{\tiny$\pm 126$} & $744$\,{\tiny$\pm 283$} & $799$\,{\tiny$\pm 186$} & $1211$\,{\tiny$\pm 69$} & \cellcolor{baselinesecond}$\underline{1243}$\,{\tiny$\pm 37$} & $1077.94^{\mathrm{r}}$\,{\tiny$\pm 42.96$} & \cellcolor{baselinebest}$\mathbf{2967.64}$\,{\tiny$\pm 12.29$} \\
 & & Poor & $954$\,{\tiny$\pm 80$} & \cellcolor{baselinesecond}$\underline{1256}$\,{\tiny$\pm 122$} & $857$\,{\tiny$\pm 73$} & $946$\,{\tiny$\pm 66$} & $1038$\,{\tiny$\pm 26$} & $912.21^{\mathrm{r}}$\,{\tiny$\pm 23.51$} & \cellcolor{baselinebest}$\mathbf{2383.11}$\,{\tiny$\pm 5.82$} \\
\cmidrule(lr){2-10}
 & \multirow{3}{*}{\textbf{2halfcheetah}} & Good & $6846$\,{\tiny$\pm 574$} & $7025$\,{\tiny$\pm 439$} & $1434$\,{\tiny$\pm 1903$} & $8246$\,{\tiny$\pm 342$} & $8514$\,{\tiny$\pm 336$} & \cellcolor{baselinesecond}$\underline{8569.00}^{\mathrm{r}}$\,{\tiny$\pm 42.59$} & \cellcolor{baselinebest}$\mathbf{12149.08}^{\ddagger}$\,{\tiny$\pm 13.68$} \\
 & & Medium & $1627$\,{\tiny$\pm 187$} & \cellcolor{baselinesecond}$\underline{2561}$\,{\tiny$\pm 82$} & $1892$\,{\tiny$\pm 220$} & $2207$\,{\tiny$\pm 23$} & $2203$\,{\tiny$\pm 65$} & $2477.22^{\mathrm{r}}$\,{\tiny$\pm 4.11$} & \cellcolor{baselinebest}$\mathbf{6744.68}^{\ddagger}$\,{\tiny$\pm 1.03$} \\
 & & Poor & $465$\,{\tiny$\pm 59$} & $736$\,{\tiny$\pm 72$} & $384$\,{\tiny$\pm 420$} & $759$\,{\tiny$\pm 18$} & $760$\,{\tiny$\pm 15$} & \cellcolor{baselinesecond}$\underline{937.09}^{\mathrm{r}}$\,{\tiny$\pm 8.23$} & \cellcolor{baselinebest}$\mathbf{1803.67}$\,{\tiny$\pm 7.50$} \\
\cmidrule(lr){2-10}
 & \multirow{3}{*}{\textbf{4ant}} & Good & $2802$\,{\tiny$\pm 133$} & $2628$\,{\tiny$\pm 971$} & $344$\,{\tiny$\pm 631$} & \cellcolor{baselinesecond}$\underline{3080}$\,{\tiny$\pm 38$} & $3068$\,{\tiny$\pm 44$} & $1750.85^{\mathrm{r}}$\,{\tiny$\pm 29.29$} & \cellcolor{baselinebest}$\mathbf{3744.73}$\,{\tiny$\pm 6.64$} \\
 & & Medium & $1617$\,{\tiny$\pm 153$} & $1843$\,{\tiny$\pm 494$} & $929$\,{\tiny$\pm 349$} & $1649$\,{\tiny$\pm 100$} & \cellcolor{baselinesecond}$\underline{1871}$\,{\tiny$\pm 52$} & $1254.43^{\mathrm{r}}$\,{\tiny$\pm 35.45$} & \cellcolor{baselinebest}$\mathbf{3095.20}$\,{\tiny$\pm 9.51$} \\
 & & Poor & $1033$\,{\tiny$\pm 122$} & $1075$\,{\tiny$\pm 96$} & $518$\,{\tiny$\pm 112$} & $1295$\,{\tiny$\pm 57$} & \cellcolor{baselinesecond}$\underline{1353}$\,{\tiny$\pm 44$} & \resizebox{.98\mujocobaselinewidth}{!}{$1010.19^{\mathrm{r}}$\,{\tiny$\pm 39.19$}} & \cellcolor{baselinebest}$\mathbf{2881.69}$\,{\tiny$\pm 45.62$} \\
\midrule[1pt]
\multicolumn{2}{c}{\multirow{2}{*}{\textbf{Scenarios}}} & \multirow{2}{*}{\textbf{Dataset}} & \multicolumn{2}{c}{Offline single-agent extensions} & \multicolumn{4}{c|}{Offline MARL} & \mbox{\textbf{Off. + On.}} \\
\cmidrule(lr){4-5}\cmidrule(lr){6-9}\cmidrule(lr){10-10}
 & & & MATD3BC & MACQL & OMAR & OMIGA & MADiff & MAC-Flow & \textbf{MA-FPPO} \\
\midrule
\multirow{12}{*}{\rotatebox[origin=c]{90}{\textbf{OMIGA}}} & \multirow{4}{*}{\textbf{HalfCheetah}} & Expert & $4401.6$\,{\tiny$\pm 169.1$} & $4589.5$\,{\tiny$\pm 98.5$} & $-206.7$\,{\tiny$\pm 161.1$} & $3383.6$\,{\tiny$\pm 552.7$} &\cellcolor{baselinesecond}$\underline{4711.4}$\,{\tiny$\pm 213.6$} & $4650.0$\,{\tiny$\pm 271.6$} & \cellcolor{baselinebest}$\mathbf{8124.03}$\,{\tiny$\pm 1.17$} \\
 & & Medium & $2620.8$\,{\tiny$\pm 69.9$} & $3189.4$\,{\tiny$\pm 306.9$} & $-265.7$\,{\tiny$\pm 147.0$} & $3608.1$\,{\tiny$\pm 237.4$} & $2650.0$\,{\tiny$\pm 365.4$} &\cellcolor{baselinesecond}$\underline{4358.5}$\,{\tiny$\pm 369.2$} & \cellcolor{baselinebest}$\mathbf{10223.11}$\,{\tiny$\pm 1.21$} \\
 & & M-R &\cellcolor{baselinesecond}$\underline{3528.9}$\,{\tiny$\pm 120.9$} & $3500.7$\,{\tiny$\pm 293.9$} & $-235.4$\,{\tiny$\pm 154.9$} & $2504.7$\,{\tiny$\pm 83.5$} & $2830.5$\,{\tiny$\pm 292.8$} & $3030.2$\,{\tiny$\pm 436.8$} & \cellcolor{baselinebest}$\mathbf{10114.65}$\,{\tiny$\pm 4.52$} \\
 & & M-E & $3518.1$\,{\tiny$\pm 381.0$} & $4738.2$\,{\tiny$\pm 181.1$} & $-253.8$\,{\tiny$\pm 63.9$} & $2948.5$\,{\tiny$\pm 518.9$} & $4410.9$\,{\tiny$\pm 836.8$} &\cellcolor{baselinesecond}$\underline{5139.9}$\,{\tiny$\pm 84.1$} & \cellcolor{baselinebest}$\mathbf{9029.89}$\,{\tiny$\pm 1.82$} \\
\cmidrule(lr){2-10}
 & \multirow{4}{*}{\textbf{Hopper}} & Expert & $3309.9$\,{\tiny$\pm 4.5$} &\cellcolor{baselinesecond}$\underline{3359.1}$\,{\tiny$\pm 513.8$} & $2.4$\,{\tiny$\pm 1.5$} & $859.6$\,{\tiny$\pm 709.5$} & $2853.3$\,{\tiny$\pm 593.8$} &\cellcolor{baselinebest}$\mathbf{3592.1}$\,{\tiny$\pm 8.9$} & $2322.14$\,{\tiny$\pm 25.98$} \\
 & & Medium & $870.4$\,{\tiny$\pm 156.7$} & $901.3$\,{\tiny$\pm 199.9$} & $21.3$\,{\tiny$\pm 24.9$} & $1189.3$\,{\tiny$\pm 544.3$} &\cellcolor{baselinesecond}$\underline{1436.8}$\,{\tiny$\pm 449.5$} & $1023.5$\,{\tiny$\pm 253.0$} & \cellcolor{baselinebest}$\mathbf{3857.48}$\,{\tiny$\pm 1.02$} \\
 & & M-R & $269.7$\,{\tiny$\pm 41.8$} & $31.4$\,{\tiny$\pm 15.2$} & $3.3$\,{\tiny$\pm 3.2$} & $774.2$\,{\tiny$\pm 494.3$} & $936.1$\,{\tiny$\pm 574.0$} &\cellcolor{baselinesecond}$\underline{1166.3}$\,{\tiny$\pm 451.9$} & \cellcolor{baselinebest}$\mathbf{3382.91}$\,{\tiny$\pm 21.37$} \\
 & & M-E &\cellcolor{baselinesecond}$\underline{2904.3}$\,{\tiny$\pm 477.4$} & $2751.8$\,{\tiny$\pm 123.3$} & $1.4$\,{\tiny$\pm 0.9$} & $709.0$\,{\tiny$\pm 595.7$} & $2810.4$\,{\tiny$\pm 723.2$} &\cellcolor{baselinebest}$\mathbf{2988.3}$\,{\tiny$\pm 480.2$} & $2386.84$\,{\tiny$\pm 49.69$} \\
\cmidrule(lr){2-10}
 & \multirow{4}{*}{\textbf{Ant}} & Expert & $2046.9$\,{\tiny$\pm 17.1$} &\cellcolor{baselinesecond}$\underline{2082.4}$\,{\tiny$\pm 21.7$} & $312.5$\,{\tiny$\pm 297.5$} & $2055.5$\,{\tiny$\pm 1.6$} & $2060.0$\,{\tiny$\pm 10.3$} & $2060.2$\,{\tiny$\pm 20.0$} & \cellcolor{baselinebest}$\mathbf{6615.92}$\,{\tiny$\pm 13.66$} \\
 & & Medium & $1422.6$\,{\tiny$\pm 21.1$} & $1033.9$\,{\tiny$\pm 66.4$} & $-1710.0$\,{\tiny$\pm 1589.0$} & $1418.4$\,{\tiny$\pm 5.4$} & $1428.4$\,{\tiny$\pm 14.7$} &\cellcolor{baselinesecond}$\underline{1432.4}$\,{\tiny$\pm 17.8$} & \cellcolor{baselinebest}$\mathbf{5640.90}$\,{\tiny$\pm 27.51$} \\
 & & M-R & $995.2$\,{\tiny$\pm 52.8$} & $434.6$\,{\tiny$\pm 108.3$} & $-2014.2$\,{\tiny$\pm 844.7$} & $1105.1$\,{\tiny$\pm 88.9$} & $1294.5$\,{\tiny$\pm 360.2$} &\cellcolor{baselinesecond}$\underline{1498.4}$\,{\tiny$\pm 20.3$} & \cellcolor{baselinebest}$\mathbf{4738.50}$\,{\tiny$\pm 9.16$} \\
 & & M-E & $1636.1$\,{\tiny$\pm 96.0$} & $1800.2$\,{\tiny$\pm 21.5$} & $-2992.8$\,{\tiny$\pm 7.0$} & $1720.3$\,{\tiny$\pm 110.6$} & $1740.2$\,{\tiny$\pm 158.9$} & \cellcolor{baselinesecond}$\underline{2053.3}$\,{\tiny$\pm 20.4$} & \cellcolor{baselinebest}$\mathbf{6704.97}$\,{\tiny$\pm 16.96$} \\
\cmidrule(lr){2-10}
\multicolumn{3}{l}{\textbf{MA-MuJoCo average}} & $2293.7$ & $2367.7$ & $-611.5$ & $1856.4$ & $2430.21$ & \cellcolor{baselinesecond}$\underline{2749.4}$ & \cellcolor{baselinebest}$\mathbf{6095.11}$\,{\tiny$\pm 2.89$} \\
\bottomrule
\end{tabular}}
\endgroup
\par\smallskip
{\footnotesize\raggedright Multi-agent Twin Delayed Deep Deterministic Policy Gradient with behavioral cloning (MATD3BC, also MA-TD3+BC) and multi-agent Conservative Q-Learning (MACQL) extend single-agent methods. Decentralized and centralized MADiff variants are MADiff-D and MADiff-C. Medium-Replay (M-R) and Medium-Expert (M-E) label dataset qualities.\par}
\end{table}

\begin{table}[!t]
\centering
\caption{\textbf{Published baselines and MA-FPPO on discrete control.} Sources and notation are in Appendix~\ref{app:table-smac}.}
\label{tab:discrete-main}
\label{tab:published_discrete}
\begingroup
\scriptsize
\renewcommand{\tiny}{\fontsize{5}{6}\selectfont}
\setlength{\tabcolsep}{1.8pt}
\renewcommand{\arraystretch}{1.14}
\resizebox{\linewidth}{!}{%
\begin{tabular*}{1.16\linewidth}{@{\extracolsep{\fill}}lllcccccccc|c@{}}
\toprule
\multicolumn{2}{c}{\multirow{2}{*}{\textbf{Scenarios}}} & \multirow{2}{*}{\textbf{Dataset}} & \multicolumn{3}{c}{Gaussian policies} & \multicolumn{3}{c}{Diffusion policies} & \multicolumn{2}{c|}{Flow policies} & \textbf{Off. + On.} \\
\cmidrule(lr){4-6}\cmidrule(lr){7-9}\cmidrule(lr){10-11}\cmidrule(lr){12-12}
 & & & BC & MABCQ & MACQL & \shortstack{Diffusion\\BC} & MADiff & DoF & Flow BC & MAC-Flow & \textbf{MA-FPPO} \\
\midrule
\multirow{12}{*}{\rotatebox[origin=c]{90}{\textbf{SMACv1}}} & \multirow{3}{*}{\textbf{3m}} & Good & $16.0$\,{\tiny$\pm 1.0$} & $3.7$\,{\tiny$\pm 1.1$} & $19.1$\,{\tiny$\pm 0.1$} & $19.5$\,{\tiny$\pm 0.5$} & $19.3$\,{\tiny$\pm 0.5$} & \cellcolor{baselinesecond}$\underline{19.8}$\,{\tiny$\pm 0.2$} & \cellcolor{baselinebest}$\mathbf{20.0}$\,{\tiny$\pm 0.0$} & \cellcolor{baselinesecond}$\underline{19.8}$\,{\tiny$\pm 0.2$} & \cellcolor{baselinebest}$\mathbf{20.00}^{*}$\,{\tiny$\pm 0.00$} \\
 & & Medium & $8.2$\,{\tiny$\pm 0.8$} & $4.0$\,{\tiny$\pm 1.0$} & $13.7$\,{\tiny$\pm 0.3$} & $13.3$\,{\tiny$\pm 0.7$} & $16.4$\,{\tiny$\pm 2.6$} & \cellcolor{baselinesecond}$\underline{18.6}$\,{\tiny$\pm 1.2$} & $14.7$\,{\tiny$\pm 1.5$} & $18.0$\,{\tiny$\pm 3.2$} & \cellcolor{baselinebest}$\mathbf{20.00}^{*}$\,{\tiny$\pm 0.00$} \\
 & & Poor & $4.4$\,{\tiny$\pm 0.1$} & $3.4$\,{\tiny$\pm 1.0$} & $4.2$\,{\tiny$\pm 0.1$} & $4.2$\,{\tiny$\pm 0.2$} & $10.3$\,{\tiny$\pm 6.1$} & \cellcolor{baselinesecond}$\underline{10.9}$\,{\tiny$\pm 1.1$} & $4.5$\,{\tiny$\pm 0.1$} & $10.6$\,{\tiny$\pm 2.2$} & \cellcolor{baselinebest}$\mathbf{19.68}^{*}$\,{\tiny$\pm 0.09$} \\
\cmidrule(lr){2-12}
 & \multirow{3}{*}{\textbf{8m}} & Good & $16.7$\,{\tiny$\pm 0.4$} & $4.8$\,{\tiny$\pm 0.6$} & $18.9$\,{\tiny$\pm 0.9$} & $19.4$\,{\tiny$\pm 0.5$} & $18.9$\,{\tiny$\pm 1.1$} & $19.6$\,{\tiny$\pm 0.3$} & $19.5$\,{\tiny$\pm 0.2$} & \cellcolor{baselinesecond}$\underline{19.7}$\,{\tiny$\pm 0.3$} & \cellcolor{baselinebest}$\mathbf{19.98}^{\dagger}$\,{\tiny$\pm 0.03$} \\
 & & Medium & $10.7$\,{\tiny$\pm 0.5$} & $5.6$\,{\tiny$\pm 0.6$} & $15.5$\,{\tiny$\pm 1.5$} & $18.6$\,{\tiny$\pm 0.6$} & $16.8$\,{\tiny$\pm 1.6$} & $18.6$\,{\tiny$\pm 0.8$} & $18.2$\,{\tiny$\pm 0.8$} & \cellcolor{baselinesecond}$\underline{19.4}$\,{\tiny$\pm 0.6$} & \cellcolor{baselinebest}$\mathbf{20.00}^{\dagger}$\,{\tiny$\pm 0.00$} \\
 & & Poor & $5.3$\,{\tiny$\pm 0.1$} & $3.6$\,{\tiny$\pm 0.8$} & $7.5$\,{\tiny$\pm 1.0$} & $4.8$\,{\tiny$\pm 0.2$} & $9.8$\,{\tiny$\pm 0.9$} & \cellcolor{baselinesecond}$\underline{12.0}$\,{\tiny$\pm 1.2$} & $4.9$\,{\tiny$\pm 0.1$} & $11.5$\,{\tiny$\pm 0.8$} & \cellcolor{baselinebest}$\mathbf{17.94}$\,{\tiny$\pm 1.17$} \\
\cmidrule(lr){2-12}
 & \multirow{3}{*}{\textbf{2s3z}} & Good & $18.2$\,{\tiny$\pm 0.4$} & $7.7$\,{\tiny$\pm 0.9$} & $17.4$\,{\tiny$\pm 0.3$} & $18.0$\,{\tiny$\pm 1.0$} & $15.9$\,{\tiny$\pm 1.2$} & $18.5$\,{\tiny$\pm 0.8$} & \cellcolor{baselinesecond}$\underline{19.5}$\,{\tiny$\pm 0.1$} & \cellcolor{baselinesecond}$\underline{19.5}$\,{\tiny$\pm 0.5$} & \cellcolor{baselinebest}$\mathbf{20.00}^{*}$\,{\tiny$\pm 0.00$} \\
 & & Medium & $12.3$\,{\tiny$\pm 0.7$} & $7.6$\,{\tiny$\pm 0.7$} & $15.6$\,{\tiny$\pm 0.4$} & $13.4$\,{\tiny$\pm 1.4$} & $15.6$\,{\tiny$\pm 0.3$} & \cellcolor{baselinesecond}$\underline{18.1}$\,{\tiny$\pm 0.9$} & $15.1$\,{\tiny$\pm 2.0$} & $17.6$\,{\tiny$\pm 0.6$} & \cellcolor{baselinebest}$\mathbf{19.99}^{*}$\,{\tiny$\pm 0.01$} \\
 & & Poor & $6.7$\,{\tiny$\pm 0.3$} & $6.6$\,{\tiny$\pm 0.2$} & $8.4$\,{\tiny$\pm 0.8$} & $6.2$\,{\tiny$\pm 1.2$} & $8.5$\,{\tiny$\pm 1.3$} & \cellcolor{baselinesecond}$\underline{10.0}$\,{\tiny$\pm 1.1$} & $6.9$\,{\tiny$\pm 0.8$} & $8.5$\,{\tiny$\pm 0.6$} & \cellcolor{baselinebest}$\mathbf{20.01}^{*}$\,{\tiny$\pm 0.01$} \\
\cmidrule(lr){2-12}
 & \multirow{3}{*}{\textbf{5m6m}} & Good & $15.8$\,{\tiny$\pm 3.6$} & $2.4$\,{\tiny$\pm 0.4$} & $16.2$\,{\tiny$\pm 1.6$} & $16.8$\,{\tiny$\pm 2.3$} & $16.5$\,{\tiny$\pm 2.8$} & $17.7$\,{\tiny$\pm 1.1$} & $14.7$\,{\tiny$\pm 2.1$} & \cellcolor{baselinebest}$\mathbf{18.6}$\,{\tiny$\pm 3.5$} & \cellcolor{baselinesecond}$\underline{17.72}$\,{\tiny$\pm 0.23$} \\
 & & Medium & $12.4$\,{\tiny$\pm 0.9$} & $3.8$\,{\tiny$\pm 0.5$} & $15.1$\,{\tiny$\pm 2.9$} & $12.5$\,{\tiny$\pm 2.1$} & $15.2$\,{\tiny$\pm 2.6$} & \cellcolor{baselinesecond}$\underline{16.2}$\,{\tiny$\pm 0.9$} & $12.8$\,{\tiny$\pm 0.8$} & $15.6$\,{\tiny$\pm 1.3$} & \cellcolor{baselinebest}$\mathbf{18.50}$\,{\tiny$\pm 0.13$} \\
 & & Poor & $7.5$\,{\tiny$\pm 0.2$} & $3.3$\,{\tiny$\pm 0.5$} & $10.5$\,{\tiny$\pm 3.1$} & $8.0$\,{\tiny$\pm 1.0$} & $8.9$\,{\tiny$\pm 1.3$} & \cellcolor{baselinesecond}$\underline{10.8}$\,{\tiny$\pm 0.3$} & $7.7$\,{\tiny$\pm 0.8$} & $9.8$\,{\tiny$\pm 2.1$} & \cellcolor{baselinebest}$\mathbf{19.13}$\,{\tiny$\pm 0.04$} \\
\midrule
\multirow{4}{*}{\rotatebox[origin=c]{90}{\fontsize{5}{6}\selectfont\textbf{SMACv2}}} & \textbf{Terran 5v5} & Replay & $7.3$\,{\tiny$\pm 1.0$} & $13.8$\,{\tiny$\pm 4.4$} & $11.8$\,{\tiny$\pm 0.9$} & $9.3$\,{\tiny$\pm 0.9$} & $13.3$\,{\tiny$\pm 1.8$} & $15.4$\,{\tiny$\pm 1.3$} & $8.3$\,{\tiny$\pm 1.9$} & \cellcolor{baselinesecond}$\underline{16.6}$\,{\tiny$\pm 4.3$} & \cellcolor{baselinebest}$\mathbf{16.91}$\,{\tiny$\pm 0.58$} \\
\cmidrule(lr){2-12}
 & \textbf{Zerg 5v5} & Replay & $6.8$\,{\tiny$\pm 0.6$} & $10.3$\,{\tiny$\pm 1.2$} & $10.3$\,{\tiny$\pm 3.4$} & $8.1$\,{\tiny$\pm 1.7$} & $10.2$\,{\tiny$\pm 1.1$} & \cellcolor{baselinesecond}$\underline{12.0}$\,{\tiny$\pm 1.1$} & $4.6$\,{\tiny$\pm 0.5$} & $9.8$\,{\tiny$\pm 1.5$} & \cellcolor{baselinebest}$\mathbf{15.56}$\,{\tiny$\pm 0.47$} \\
\cmidrule(lr){2-12}
 & \textbf{Terran 10v10} & Replay & $7.4$\,{\tiny$\pm 0.5$} & $12.7$\,{\tiny$\pm 2.0$} & $11.8$\,{\tiny$\pm 2.0$} & $5.5$\,{\tiny$\pm 1.5$} & $13.8$\,{\tiny$\pm 1.3$} & \cellcolor{baselinesecond}$\underline{14.6}$\,{\tiny$\pm 1.1$} & $5.8$\,{\tiny$\pm 1.7$} & $13.0$\,{\tiny$\pm 4.7$} & \cellcolor{baselinebest}$\mathbf{14.98}$\,{\tiny$\pm 0.32$} \\
\cmidrule(lr){2-12}
\multicolumn{3}{l}{\textbf{SMACv2 average}} & $7.2$ & $12.3$ & $11.3$ & $7.6$ & $12.4$ & \cellcolor{baselinesecond}$\underline{14.0}$ & $6.2$ & $13.1$ & \cellcolor{baselinebest}$\mathbf{15.82}$\,{\tiny$\pm 0.23$} \\
\bottomrule
\end{tabular*}}
\endgroup
\par\smallskip
{\footnotesize\raggedright Multi-agent Batch-Constrained Q-learning (MABCQ) extends batch-constrained learning to multiple agents.\par}
\end{table}

\subsubsection{Team Advantages and Policy Updates}
A new centralized critic $V_\psi(s_t)$ supplies shared, normalized advantages $\widetilde A_t$ computed using generalized advantage estimation (GAE)~\citep{gae} and is fitted to detached return targets $\widehat R_t$. Appendix~\ref{app:online-details} specifies the targets and termination handling.

Use $(\mu_\theta,w_t^i)=(\pi_\theta,a_t^i)$ for discrete actions and $(q_\theta,u_t^i)$ for continuous actions. The collection parameters $\theta_{\mathrm{old}}$ remain fixed during each rollout update. Online parameters $\theta$ include the student and, for continuous actions, $\log\sigma_\theta$. With binary mask $m_t^i$ indicating whether agent $i$ is active at step $t$, define
\begin{equation}
 \rho_t^i(\theta)=\frac{\mu_\theta(w_t^i\mid h_t^i)}{\mu_{\theta_{\mathrm{old}}}(w_t^i\mid h_t^i)},
 \qquad Z=\max\!\left(1,\sum_{t,i}m_t^i\right).
 \label{eq:ratio}
\end{equation}
The PPO objective averages over active agents:
\begin{align}
 L_{\mathrm{clip}}(\theta)
 &=\frac{\sum_{t,i}m_t^i\min\!\left(\rho_t^i\widetilde A_t,
 \clip(\rho_t^i,1-\epsilon,1+\epsilon)\widetilde A_t\right)}
 {Z},\label{eq:ppo}\\
 \mathcal L_{\mathrm{actor}}
 &=-L_{\mathrm{clip}}+\beta\mathcal R_{\mathrm{ref}}-\eta\mathcal H,
 \label{eq:actor}
\end{align}
where $\epsilon>0$ is the PPO clipping parameter, $\beta\geq0$ weights reference KL, $\eta\geq0$ weights entropy, $\mathcal H$ is the entropy of the sampling policy and $\mathcal R_{\mathrm{ref}}$ optionally penalizes deviation from a frozen copy of the initial online actor.
All continuous agents are active. Log probabilities are summed across each agent's action coordinates before forming the likelihood ratio. The centralized critic minimizes $\frac12\E[(V_\psi-\widehat R)^2]$.

Setting $\beta=0$ removes reference KL while retaining PPO clipping and the KL stopping rule from the old to the current policy. Appendix~\ref{app:online-details} specifies the update rules and hyperparameters.

Algorithms~\ref{alg:offline}--\ref{alg:inference} detail the procedures. Appendix~\ref{app:proofs} gives the theoretical analysis.

\section{Experiments}
\definecolor{experimentquestion}{RGB}{176,78,12}
\newcommand{\exprq}[2]{\noindent\textcolor{experimentquestion}{\textbf{RQ#1. #2}}}
\newcommand{\expanswer}[2]{\noindent\textbf{A#1.} \textcolor{experimentquestion}{\textbf{#2}}}
We test broad applicability across environments and offline data qualities, focusing on coordination in situations not covered by the offline data. We examine online improvement, pretraining alternatives, and optimization components through three research questions (RQs):
\nopagebreak[4]

\exprq{1}{Can online fine-tuning improve coordination over offline pretraining?}

\exprq{2}{Do flow pretraining and the Gaussian online policy each improve final returns over the tested alternatives?}

\exprq{3}{How do reference KL regularization and centralized value estimation affect online fine-tuning?}

\subsection{Experimental Setup}

We evaluate discrete control on the StarCraft Multi-Agent Challenge (SMAC)~\citep{smac} and SMAC version 2 (SMACv2)~\citep{smacv2}, and continuous control on Multi-Agent Particle Environments (MPE)~\citep{mpe,omar} and Multi-Agent MuJoCo (MA-MuJoCo)~\citep{facmac}. The OG-MARL and OMIGA MuJoCo collections~\citep{ogmarl,omiga} are reported separately. Tables~\ref{tab:continuous-main} and~\ref{tab:discrete-main} include imitation, offline reinforcement learning (RL), and generative policy baselines reported by MAC-Flow, MADiff, and DOM2~\citep{macflow,madiff,dom2}. Baselines vary because prior studies cover different benchmarks. Some methods report no results on our selected datasets and have no publicly available code, preventing reproduction of the missing results. $^{\mathrm r}$ marks locally pretrained students. \method{} uses 1M offline updates followed by online fine-tuning, with online budgets marked in the tables. The ablations use four settings and 50M online steps. Main-table MPE scores are normalized. Other results use raw returns. Appendix~\ref{app:experimental-protocol} specifies tasks, sources, and budgets. We use three training seeds, with error bands computed across the three seeds. Final results are evaluated over 200 episodes per seed, while learning curves are evaluated over 20 episodes per seed.

\begin{figure}[!t]
\centering
\includegraphics[width=\linewidth]{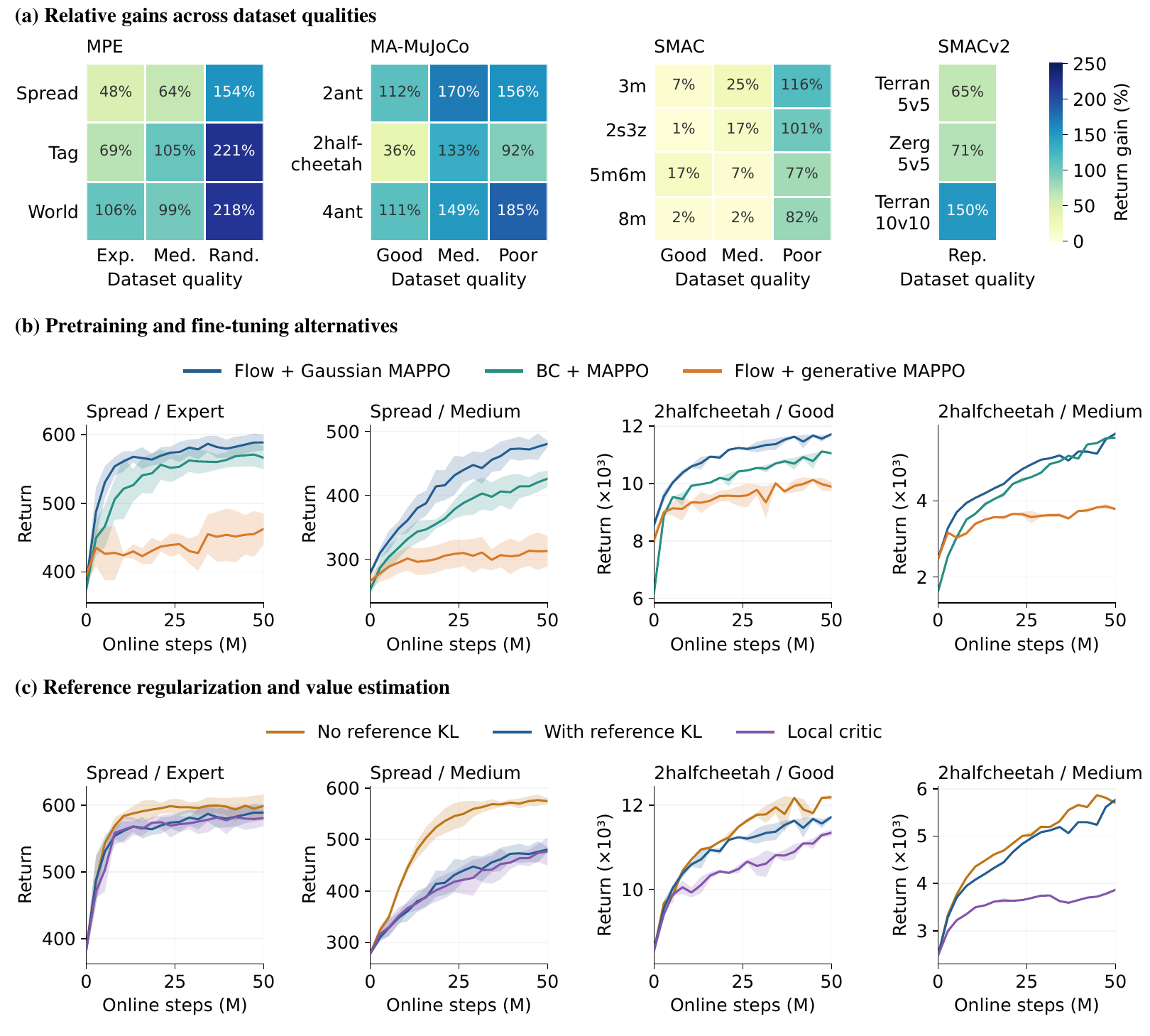}
\caption{\textbf{Online gains and learning dynamics.} (a) Percentage gains over each pretrained policy across dataset qualities. (b) Learning curves for Gaussian MAPPO initialized from flow pretraining or BC, and generative MAPPO initialized from flow pretraining. (c) Ablations of reference KL and the local critic. In (b,c), columns show Spread Expert/Medium and 2halfcheetah Good/Medium. Appendix Figure~\ref{fig:ablation-endpoints} shows endpoint bars.}
\label{fig:dataset-quality-gain}
\label{fig:parameterization-combined}
\label{fig:zero-step-ablation}
\label{fig:parameterization-ablation}
\label{fig:component-ablation}
\end{figure}

\begin{figure}[!t]
\centering
\includegraphics[width=\linewidth]{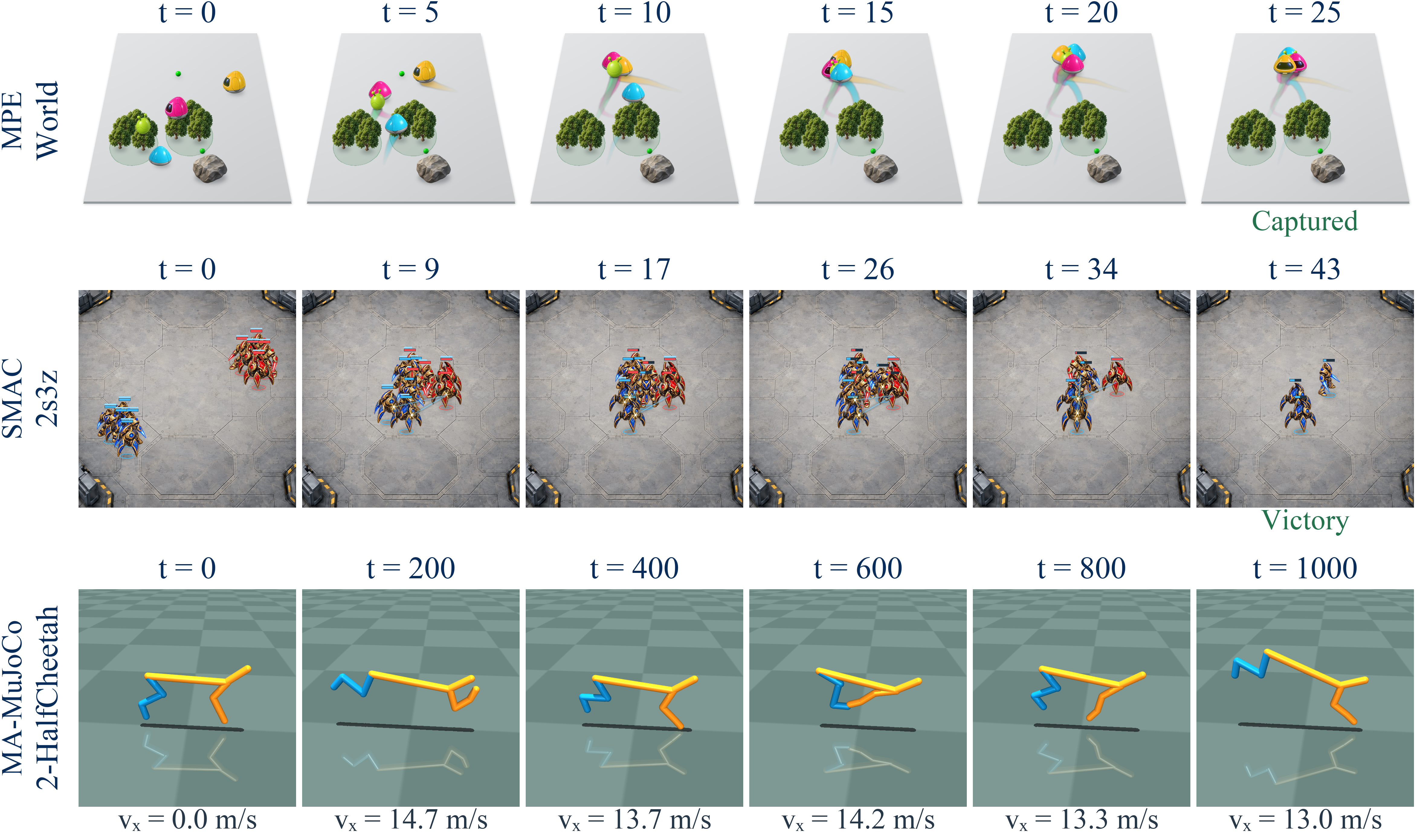}
\caption{\textbf{Actual MA-FPPO rollouts.} Six recorded states show capture in World, victory with no surviving enemies in 2s3z, and locomotion in 2-HalfCheetah. $t$: environment step. $v_x$: forward velocity in m/s. Poses follow recorded trajectories. Trails and combat effects illustrate movement and actions.}
\label{fig:main-rollouts}
\end{figure}

\subsection{Results and Research Questions}
\expanswer{1}{Online fine-tuning improves the coordination learned offline.}
Across the 30 settings , \method{} improves on the strongest listed offline reference by an average of 52.8\%, with a median gain of 34.7\%. Settings receive equal weight. MA-FPPO attains the highest displayed mean, including ties, in 13 of 14 discrete settings, eight of nine MPE settings, and all seven OG-MARL settings. SMACv2 averages 15.82, compared with the strongest offline average of 14.0. The exceptions are Spread Expert, where DOM2 has the highest displayed mean, and 5m6m Good, where \method{} scores 17.72 versus 18.6 for MAC-Flow.

\textbf{Combining offline experience and online interaction.} Appendix Figures~\ref{fig:main-training-curves} and~\ref{fig:main-smac-winrate} compare continuous and discrete control. Pretraining segments show our offline MAC-Flow reproduction. Black dashed curves show purely online MAPPO with the same architecture and random initialization. Over shared online budgets, the return and win-rate curves show how flow pretraining provides useful initialization and online interaction improves beyond offline performance, with benefits varying across tasks and data qualities. Offline collection and pretraining costs are excluded from the online comparison. All 13 paired endpoints with local offline reproductions improve, averaging 129.8\%. Figure~\ref{fig:main-rollouts} illustrates successful episodes.

\textbf{Effect of data quality.} \method{} substantially improves policies pretrained on low-quality datasets, demonstrating that online fine-tuning can overcome performance limitations of offline pretraining.
Relative gains over the strongest listed offline reference are larger for Random than Expert MPE data, and for Poor than Good SMAC data. Appendix~\ref{app:relative-gains} provides the full breakdown and aggregation protocol.

Figure~\ref{fig:dataset-quality-gain} reports gains over each pretrained policy. On Poor data, 2halfcheetah, 4ant, and 8m gain 92\%, 185\%, and 82\%, respectively. Lower-quality data often leave more room for improvement, but the trend is not monotonic and percentage gains depend on the starting return.

\expanswer{2}{Flow pretraining and the Gaussian online policy each improve final mean returns.}
\label{sec:zero-step-ablation}
\label{sec:parameterization-ablation}
Figure~\ref{fig:ablation-endpoints}(a) compares three combinations under the same 50M online budget. \textbf{Choice of pretraining.} Holding the Gaussian policy and online optimization algorithm fixed compares flow and BC initialization. Flow pretraining gives higher final means in all four settings: 577.02 versus 568.19 on Spread Expert, 471.99 versus 416.78 on Spread Medium, 11712.77 versus 10989.10 on HalfCheetah Good, and 5764.36 versus 5654.21 on HalfCheetah Medium. The largest relative gap occurs on Spread Medium. The comparison with BC evaluates the complete initialization from flow pretraining, including value guidance and distillation.

\textbf{Choice of online policy.} Keeping flow pretraining, the Gaussian policy also outperforms generative MAPPO in all four endpoint comparisons. On Spread Medium, the means are 472.0 versus 311.3. On HalfCheetah Medium, the means are approximately 5.76k versus 3.79k. The comparisons support inheriting flow pretrained behavior and using Gaussian exploration with explicit action likelihoods for online fine-tuning. Appendix~\ref{app:zero-step-ablation} defines the policies. The observed advantage is consistent with the mechanism in Appendix~\ref{app:latent-return} without isolating its contribution.

\textbf{Training from scratch.} Figure~\ref{fig:main-training-curves} compares the complete procedure with randomly initialized MAPPO using the same actor architecture. Appendix~\ref{app:quality-gain} specifies budgets and reference KL configurations. The comparison evaluates the complete pretraining and online fine-tuning procedure.

\expanswer{3}{Reference KL regularization has mixed effects. The centralized critic improves all four settings.}
\textbf{Reference KL regularization.}
\label{sec:components}
Figure~\ref{fig:ablation-endpoints}(b) removes reference KL with the centralized critic fixed. Final means rise in three settings. Spread Medium improves from 471.99 to 577.77, an increase of 22.4\%. HalfCheetah Medium slightly favors reference KL, 5764.36 versus 5710.82. Reference KL regularization thus has effects that differ across tasks. The main tables retain $\beta=0.01$.

\textbf{Centralized value estimation.}
Figure~\ref{fig:ablation-endpoints}(b) shows that replacing the centralized critic with a local critic lowers all four final means when reference KL is retained. On HalfCheetah Medium, the return falls from 5764.36 to 3867.10, a decrease of 32.9\%. On Spread Expert, the return falls from 577.02 to 573.28. Joint information for value estimation helps most on the tested HalfCheetah Medium setting. Actors remain decentralized in both variants.

\section{Conclusion and Limitations}
We presented \method{}, which initializes online policies from pretrained behavior and optimizes the student using explicit action likelihoods. Our method achieves average relative gains of 52.8\% over the strongest listed offline baselines across 30 settings and 29.8\% over purely online learning across 38 comparisons with matched online budgets and evaluation protocols. Flow pretraining outperforms behavioral cloning pretraining when both are followed by online fine-tuning with Gaussian policies in all four tested comparisons. After flow pretraining, online fine-tuning with Gaussian policies also outperforms directly fine-tuning the flow policy. MA-FPPO substantially improves policies pretrained on low-quality datasets, showing that online interaction can overcome performance limitations of offline pretraining. Removing reference KL improves three of four component comparisons.

\paragraph{Limitations.}
 Offline references receive no online interaction, while the comparison with training from scratch also differs in reference KL regularization. The comparison with behavioral cloning pretraining evaluates complete pretraining procedures, including distillation and value guidance. The fixed latent policy omits the full generative action distribution, and comparisons cover the tested implementations. A performance guarantee for the complete algorithm remains unestablished. We hypothesize that fixing the latent and controlling exploration separately aid optimization and coordination, but Figure~\ref{fig:direct-flow-policy-finetuning} does not isolate the causal effects of latent randomness or multimodality.

\clearpage
\noindent\textbf{AI use.} We used AI tools only to polish the language of the manuscript. All technical content, experiments, analyses, claims, and final wording were reviewed and approved by the authors.\par
\bibliography{references}
\bibliographystyle{plainnat}
\clearpage
\appendix
\addtocontents{toc}{\protect\setcounter{tocdepth}{2}}
\begingroup
\renewcommand{\contentsname}{Appendix Contents}
\pdfbookmark[0]{Appendix Contents}{appendix-contents}
\tableofcontents
\endgroup
\clearpage
\section{Experimental Details and Results}
\label{app:experimental-protocol}

\subsection{Experimental Setup}
\label{app:baseline-attribution}
We study discrete and continuous cooperative control. \textbf{SMAC}~\citep{smac} includes 3m, 8m, 2s3z, and 5m6m with Good/Medium/Poor datasets. \textbf{SMACv2}~\citep{smacv2} includes Terran 5v5, Zerg 5v5, and Terran 10v10 with Replay datasets. \textbf{MPE}~\citep{mpe,omar} includes Spread, Tag, and World with Expert/Medium/Random datasets, plus Medium-Replay for Spread. \textbf{MA-MuJoCo}~\citep{facmac} partitions a robot's joints among agents. The OG-MARL collection~\citep{ogmarl} includes 2ant with Good/Medium/Poor data and 2halfcheetah and 4ant with Good/Medium data. Table~\ref{tab:mujoco_combined} separately lists the OMIGA collection~\citep{omiga} used by MAC-Flow~\citep{macflow}: 6-HalfCheetah, 3-Hopper, and 2-Ant, each with Expert/Medium/Medium-Replay/Medium-Expert data. The two collections are reported separately. MPE uses the OMAR environment implementation compatible with the dataset~\citep{omar}. Tag controls the predator team against a fixed pretrained prey policy. Each initialization uses 1M offline optimizer updates. The online budgets are specified below.

The primary comparison is the \emph{same student's deployment policy} before and after online fine-tuning. Discrete deployment selects the available action with the highest logit. Continuous deployment executes the clipped mean. Separately, we evaluate the stochastic sampling policy used during PPO. The checkpoint at the end of the training budget is evaluated without selecting the best checkpoint on the test set. We report SMAC win rates and raw episodic returns, and raw team returns for continuous tasks. MPE returns average over the three controlled agents. The ablations report raw returns and Table~\ref{tab:published_continuous} reports normalized scores. Comparisons of raw returns are made within each environment. SMAC's offline recurrent encoder and action masking, MuJoCo timeout handling, and MPE's fixed opponent and reward implementation are part of the protocol. Comparisons across studies require matching the stated environment and evaluation choices.

Good, Medium, Poor, Expert, and Replay identify dataset quality categories. SMACv1 uses the original SMAC environment. SMACv2 uses maps with randomized unit capabilities and separate data. MPE uses the original reward and action conversion compatible with the dataset. Tag and World control predators while holding the prey policy fixed. The 2ant environment uses the Ant-v2 two-by-four partition. Rewards are raw returns, which can be nonzero even for a random policy. Differences in environments and rewards prevent direct comparison of absolute returns across benchmark families.

Figure~\ref{fig:offline-dataset-returns} characterizes the offline data through undiscounted trajectory returns. The analysis covers 16 tasks and 48 groups defined by task and data quality, including available quality groups whose online results are not reported in the main tables. MPE rewards are averaged over the three controlled agents, with trajectories ending when all agents terminate or after 25 steps. SMAC, SMACv2, and both MuJoCo collections use the stored trajectory boundaries, counting the shared team reward once per step. The figure contains 853,065 trajectories. SMACv2 includes Replay only. The OG-MARL and OMIGA MuJoCo collections are labeled separately. OMIGA contributes 6-HalfCheetah, 3-Hopper, and 2-Ant with Expert, Medium, Medium-Replay, and Medium-Expert data, totaling 25,872 trajectories. The converted OMIGA vault data exclude reset delimiters with zero reward before the calculation of trajectory returns. The distributions show variation and overlap within quality categories and provide context for the initialization data. Offline return distributions do not evaluate learned policies.

\begin{figure}[!htb]
\centering
\includegraphics[width=\linewidth]{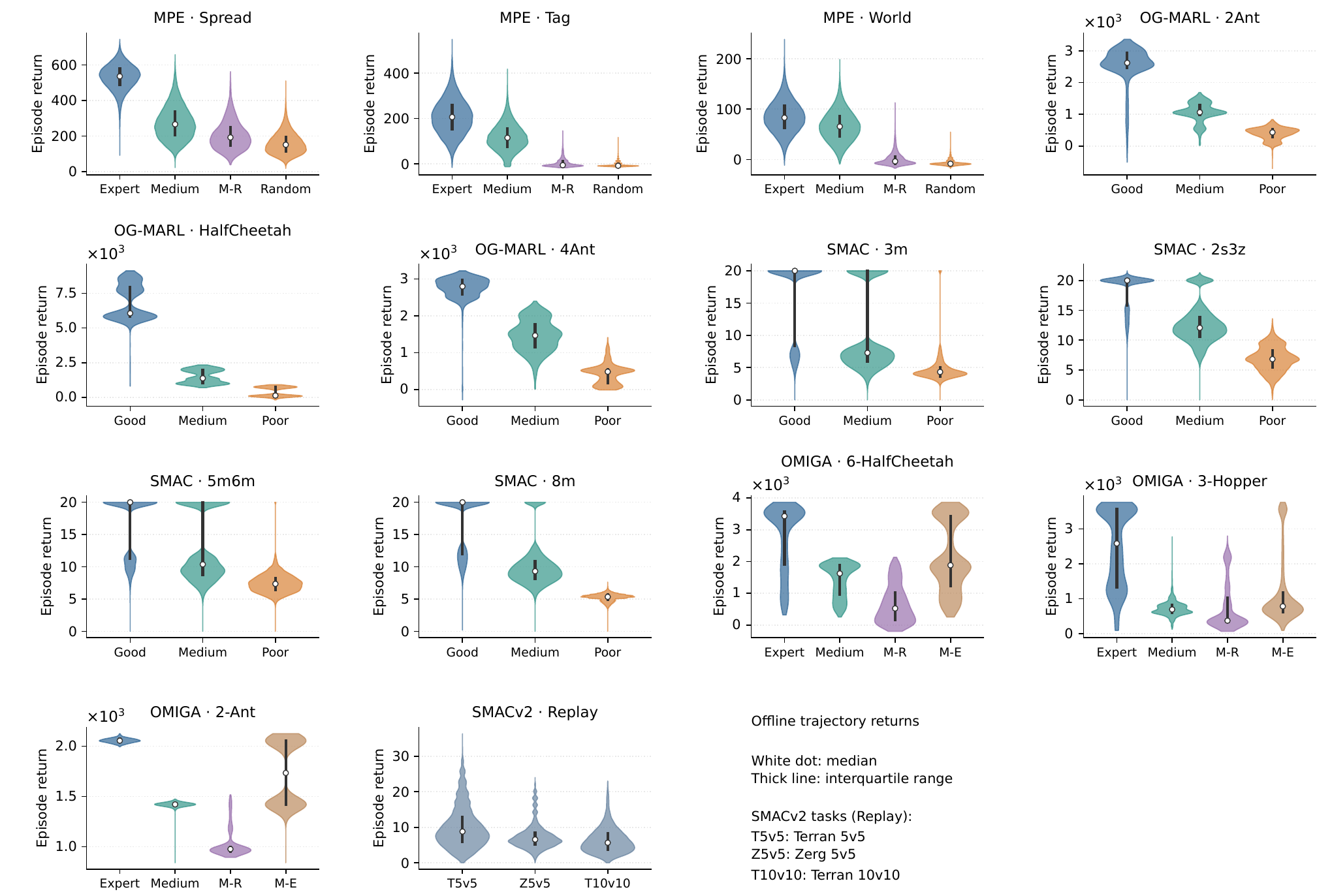}
\caption{\textbf{Offline distributions of trajectory returns.} Each panel shows one task, except the combined SMACv2 panel, which compares three Replay datasets. White dots mark medians and thick lines mark interquartile ranges. Violins use Gaussian kernel density estimates with Scott's bandwidth, restricted to the observed return range and normalized to equal maximum width within each panel. Width represents density rather than sample count. M-R and M-E denote Medium-Replay and Medium-Expert, respectively. Returns use raw scales of the respective tasks.}
\label{fig:offline-dataset-returns}
\end{figure}

For a figure showing pretraining and online fine-tuning, the final offline evaluation and evaluation at the start of online fine-tuning use matching evaluation and deployment conventions. Offline steps count gradient updates. Online steps count environment transitions. Axes of equal width display the two stages using optimizer updates and environment transitions, respectively. Saved random initialization is used when available. The original SMACv2 offline zero point was reconstructed using the initialization procedure and is labeled as reconstructed. No line is extended beyond the latest observed online evaluation.

BC denotes supervised behavioral cloning~\citep{pomerleau1991efficient}. MABCQ and MACQL are multi-agent adaptations of Batch-Constrained Q-learning (BCQ)~\citep{bcq} and Conservative Q-Learning (CQL)~\citep{kumar2020conservative}. MA-TD3+BC and MATD3BC denote adaptations of Twin Delayed Deep Deterministic Policy Gradient (TD3) with BC~\citep{fujimoto2021minimalist}. We cite the originating algorithms separately from the published multi-agent implementation and numerical sources~\citep{macflow,madiff}. ICQ, OMAR, and OMIGA refer to the original offline multi-agent methods~\citep{yang2021believe,omar,omiga}. Decentralized multi-agent diffusion (MADiff-D) and centralized multi-agent diffusion (MADiff-C) are the two MADiff variants~\citep{madiff}. DoF is the method for diffusion factorization of \citet{li2025dof}. Diffusion BC and Flow BC are the baselines trained only by imitation defined in MAC-Flow~\citep{macflow}, built on denoising diffusion~\citep{ddpm} and flow matching~\citep{fm}. The locally reproduced MAC-Flow entries retain the same method attribution~\citep{macflow}. The BC + MAPPO control combines behavioral cloning with MAPPO~\citep{pomerleau1991efficient,mappo}. Generative MAPPO and the variants with reference KL or a local critic are controls defined in this paper.

\begin{table}[ht]
\centering
\caption{Online fine-tuning settings. Online budgets count joint environment transitions, with one transition per joint environment step.}
\label{tab:hyper}
\small
\begingroup
\renewcommand{\arraystretch}{1.15}
\renewcommand{\tabularxcolumn}[1]{m{#1}}
\begin{tabularx}{\linewidth}{@{}>{\raggedright\arraybackslash}m{.40\linewidth}>{\centering\arraybackslash}X>{\centering\arraybackslash}X@{}}
\toprule
Parameter & SMAC / SMACv2 & MPE / MA-MuJoCo\\
\midrule
Actor learning rate & $2\times10^{-5}$ & $2\times10^{-5}$\\
Critic learning rate & $3\times10^{-4}$ & $3\times10^{-4}$\\
Discount / GAE $\lambda$ & 0.99 / 0.95 & 0.99 / 0.95\\
PPO clipping / KL stopping threshold & 0.05 / 0.02 & 0.05 / 0.02\\
Reference KL / entropy coefficients & 0.01 / 0.001 & 0.01 / 0.001\\
Update epochs / minibatch size & 4 / 128 & 4 / 128\\
Parallel environments / rollout length & 120 / 128 & 120 / 128\\
Gradient norm clipping & 0.5 & 0.5\\
Critic warmup & 2640 transitions & 2640 transitions\\
Reference KL direction & reference $\|$ current & current $\|$ reference\\
Initial exploration & categorical, $T=1$ & Gaussian, $\sigma=0.2$\\
Local actor input & frozen recurrent features & current local observation\\
Centralized critic & state, 128--128 multilayer perceptron (MLP) & joint observation/state, 256--256 MLP\\
\bottomrule
\end{tabularx}
\endgroup
\end{table}
Continuous environment simulation uses four worker processes per training group. With hardware shared across experiments, we compare performance by budget of environment transitions rather than wall-clock time.

Before an actor update, the stored likelihood is recomputed using the exact collected input, mask, and raw sample. The implementation checks a maximum absolute difference in log probabilities below 0.002. The first 2640 environment transitions train only the new centralized critic. The warmup fits the critic before actor updates begin.

At the final budget boundary, only the required parallel environments advance. For 50M transitions with the present rollout schedule, the final partial segment contains 80 actual environment transitions. The remaining environments retain physical state, random number generator state, and recurrent state. Advancing only the required environments matches simulator interactions to the specified budget.

\subsection{Main Results and Reporting Conventions}
\label{app:table-smac}
\label{app:table-mpe}
\label{app:table-mujoco}
\label{app:relative-gains}
\label{app:quality-gain}
Additional comparisons by task illustrate the trends in the main tables. On 3m Poor, return reaches 19.68 versus 4.4 for BC and 10.9 for the strongest listed offline baseline. 5m6m Poor reaches 19.13 versus 10.8. Tag Medium increases from the local pretrained student's 80.47 to 159.91, and World Medium from 89.82 to 172.64. In OG-MARL MuJoCo, 2ant Medium increases from 1077.94 to 2967.64 and 4ant Medium from 1254.43 to 3095.20.

Figure~\ref{fig:dataset-quality-gain} shows that initialization from lower quality data generally leaves more room for online improvement. In MPE, Random datasets yield gains of 154--221\%, compared with 48--106\% for Expert datasets. In SMAC, Poor datasets yield 77--116\% gains on 3m, 2s3z, and 5m6m, exceeding gains from Good and Medium data. All three MA-MuJoCo tasks also improve more in relative terms from Medium than from Good data. The relationship is not strictly monotonic: for example, 2ant gains 156\% from Poor data versus 170\% from Medium data. The observed gains support a tendency toward larger relative gains from weaker initial policies, rather than a universal ordering. A lower starting return also increases the percentage gain for a given absolute improvement. The SMACv2 panel with Replay data only demonstrates improvement but does not test a quality trend. The Poor entries for 2halfcheetah, 4ant, and 8m use mean returns across evaluations before and after fine-tuning. Their gains are computed as 100 times the difference in mean return divided by the pretrained mean. The 2halfcheetah Good and Medium cells retain the 50M endpoints for comparison with Poor data.

For SMAC and SMACv2, unmarked baseline values are reproduced from MAC-Flow \citep{macflow}, arXiv v2, Table 1. Offline baselines and our results of pretraining followed by online fine-tuning use different budgets. All MA-FPPO runs use 1M offline updates. $^{*}$ denotes 15M online environment steps. $^{\dagger}$ denotes training stopped at 31.122M transitions. Unmarked entries use 50M. --: result not yet available or incomplete aggregate. Unfilled baseline cells remain blank. Light green shading and boldface mark the largest displayed mean. Light gray shading and underlining mark the second-largest distinct mean, including MA-FPPO where available. Equal displayed means share ranks. Ranks describe the displayed means under the listed training budgets. 
Values are raw team episode returns. The SMACv2 average covers all three Replay tasks. No partial SMACv1 average is reported.

The Diffusion Offline Multi-agent Model (DOM2) results are from Table 1 of \citet{dom2}. Cooperative Navigation and Predator Prey correspond to Spread and Tag. The source uses its own collected datasets. DOM2 is therefore a published reference under a different data collection protocol, rather than a controlled comparison on identical offline data. 
For MPE, baselines other than DOM2 and marked entries are from MAC-Flow \citep{macflow}, arXiv v2, Table 2. Entries marked $^{\mathrm{m}}$ are from MADiff \citep{madiff}, Table 1.  Offline baselines and MA-FPPO use different training budgets. All MA-FPPO runs use 1M offline updates and the displayed 50M online endpoints. --: result not yet available or incomplete aggregate. Unfilled baseline cells remain blank. Light green shading and boldface mark the largest displayed mean. Light gray shading and underlining mark the second-largest distinct mean, including MA-FPPO where available. Equal displayed means share ranks. Ranks describe the displayed means under the listed training budgets. $^{\mathrm{r}}$: our MAC-Flow student after 1M offline updates, before RL. The score scale matches the corresponding MA-FPPO entry.
M-R/M-E: Medium-Replay/Medium-Expert. $^{\dagger}$: normalized MPE scores. Spread uses $100(R-159.8)/(516.8-159.8)$, where $R$ is the raw mean deployment return. Our raw returns averaged over evaluations are 577.02 for Expert, 471.99 for Medium, and 618.76 for Random. Tag and World also use normalized MPE scores $100(R-R_{\mathrm{random}})/(R_{\mathrm{expert}}-R_{\mathrm{random}})$, where $R_{\mathrm{expert}}$ and $R_{\mathrm{random}}$ are the expert and random reference returns, with reference returns of 185.6 and -4.1 for Tag and 79.5 and -6.8 for World. Scores exceed 100 when returns exceed the expert reference return. The offset affects only the means. $^{\mathrm{m}}$ in the MADiff column specifically denotes decentralized MADiff-D. Other MADiff entries retain the variant reported by MAC-Flow. Published Spread averages retain the original coverage. The MA-FPPO Spread average gives equal weight to Expert, Medium, Medium-Replay, and Random.

OG-MARL and OMIGA use separate baseline headers. OG-MARL: Published baselines follow MADiff \citep{madiff}, Table 1. OMIGA: baselines follow MAC-Flow \citep{macflow}, Table 2. All values are raw team returns.

OG-MARL baselines are from MADiff~\citep{madiff}, Table 1, including the Poor datasets for 2halfcheetah and 4ant.  $^{\mathrm{r}}$ denotes our MAC-Flow student after 1M offline updates, before online RL.  MA-FPPO uses 50M online steps, except $^{\ddagger}$ for 2halfcheetah, which continues from 50M to 100M. The 2ant Good and Medium entries use fresh 50M runs from the original pretrained checkpoints. At 50M, 2halfcheetah Good and Medium score $11712.77\pm16.94$ and $5764.36\pm1.77$, respectively, under the same evaluation protocol. The 100M entries have a larger online budget. MADiff-D is decentralized. MADiff-C uses centralized execution and is a separate reference. OMIGA uses 6-HalfCheetah, 3-Hopper, and 2-Ant. The MA-MuJoCo average for MA-FPPO gives equal weight to the 12 OMIGA task and dataset combinations. Published baseline averages retain their reported values. M-R/M-E: Medium-Replay/Medium-Expert. Light green shading and boldface mark the largest displayed mean. Light gray shading and underlining mark the second-largest distinct mean, as in the discrete control and MPE panels. Numerical ranks order the displayed means under the stated source protocols and budgets. Offline baselines and MA-FPPO have different budgets. The two panels use distinct collections of tasks and datasets.

Let $\mathcal C$ contain the 30 SMAC, SMACv2, MPE, and OG-MARL rows with completed \method{} scores in Tables~\ref{tab:continuous-main} and~\ref{tab:discrete-main}. The separate OMIGA collection and Spread Medium-Replay, 2halfcheetah Poor, 4ant Poor, and 8m Poor are excluded from this aggregate. For each row $c$, let $S_c$ be the displayed \method{} mean and $B_c$ the largest displayed offline mean in the same row. The comparison includes published references and local offline reproductions marked $^{\mathrm r}$. A centralized reference is included where listed. The selected reference can vary across rows. We compute
\begin{equation}
 \overline G_{\mathcal C}
 =\frac{1}{|\mathcal C|}\sum_{c\in\mathcal C}
 100\frac{S_c-B_c}{B_c}.
\end{equation}
All selected $B_c$ are positive. Each setting receives equal weight. Published average rows and settings without a completed online score are excluded. We use displayed rounded means so the calculation can be reproduced from the tables. The resulting average is 52.8\%, and the median relative gain is 34.7\%. The aggregation gives equal weight to the relative gain in each setting. SMAC/SMACv2 and MA-MuJoCo use episodic returns. MPE uses the normalized scores defined in Appendix~\ref{app:table-mpe}. MPE percentage gains are relative changes in normalized score and depend on the normalization offset. Relative gains are descriptive comparisons under the displayed budgets, including 100M online transitions for the two HalfCheetah rows. The online method receives additional environment interactions. Published references retain the source protocols.

\begin{figure}[!t]
\centering
\includegraphics[width=\linewidth]{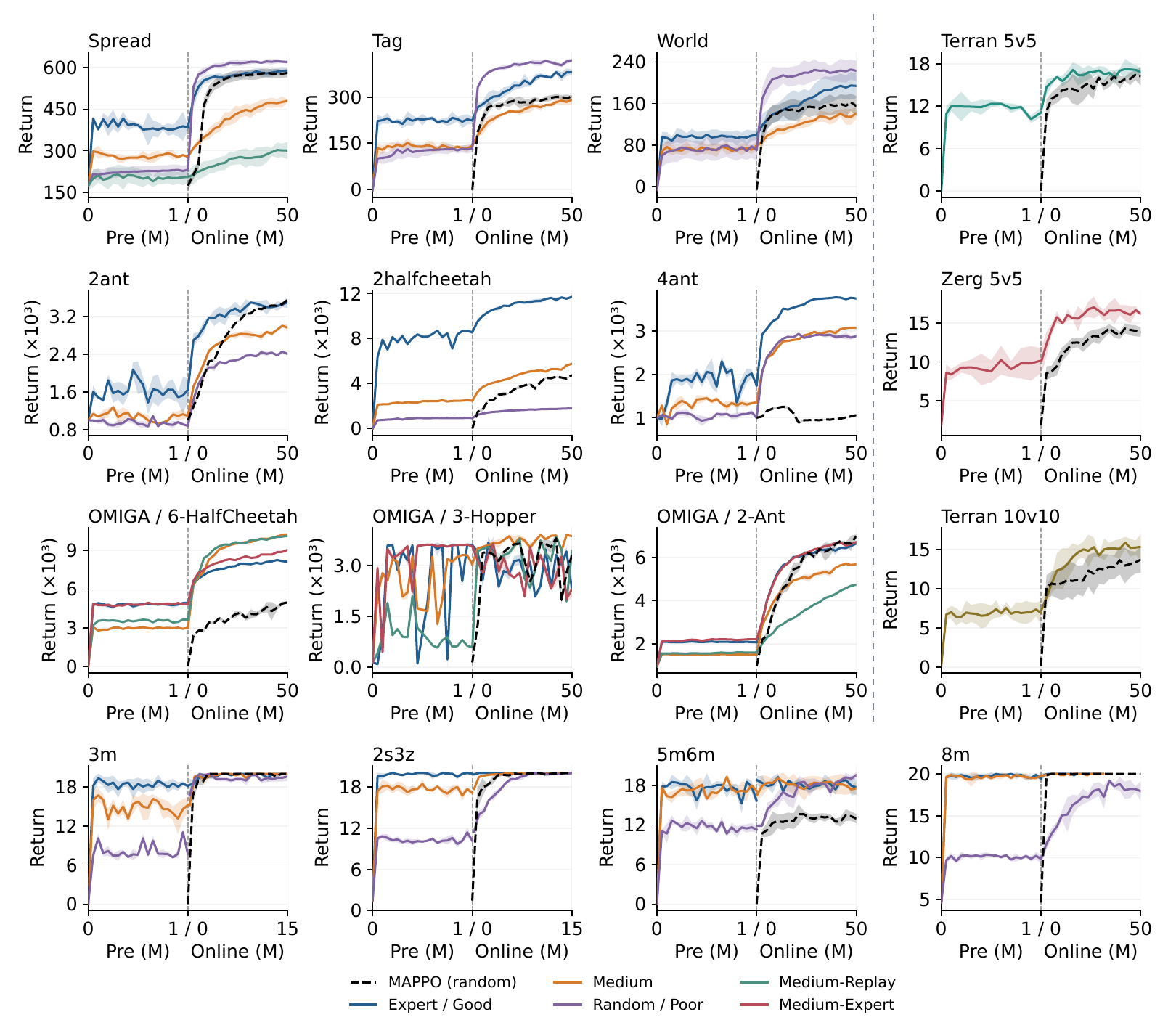}
\caption{\textbf{Learning curves during pretraining and online fine-tuning.} Sixteen tasks are arranged in four rows. Each benchmark occupies a separate row, with all four SMAC tasks shown together. OMIGA environments are separate from OG-MARL. Vertical dashed lines separate offline updates and online environment steps, both measured in millions. Dashed curves add MAPPO with the same architecture and random initialization without reference KL, online only, shown as black dashed curves in all panels. Budgets differ. All 16 MAPPO runs reach 50M online transitions. The 3m and 2s3z panels show the first 15M online steps. The 2halfcheetah panel shows the first 50M steps for all methods.}
\label{fig:main-training-curves}
\end{figure}

Across the completed comparisons outside OMIGA in Figure~\ref{fig:main-training-curves}, \method{} achieves final returns comparable to or higher than MAPPO trained from random initialization. MAPPO performs well on some tasks but does not consistently match \method{} across tasks.

By benchmark and dataset quality, the mean relative gains are: SMAC Good $-0.2\%$ across four tasks, Medium $+8.8\%$ across four tasks, and Poor $+85.9\%$ across three tasks. SMACv2 Replay gains $+11.4\%$ across three tasks. MPE gains are Expert $+45.1\%$, Medium $+46.2\%$, and Random $+150.3\%$, each across three tasks. MA-MuJoCo gains are Good $+25.1\%$ across three tasks, Medium $+122.5\%$ across three tasks, and Poor $+89.7\%$ for one task. The MPE groups additionally gain 57.72, 44.04, and 124.59 normalized score points on average, respectively. Coverage differs between groups. Spread Medium-Replay, 2halfcheetah Poor, 4ant Poor, and 8m Poor are excluded from these aggregates. The separate OMIGA MuJoCo collection is also excluded to retain the same benchmark coverage.

The comparison against the strongest reference is distinct from a before/after comparison with our own initialization. The 13 rows with local offline reproductions all improve after online fine-tuning, with a mean relative score gain of 129.8\%. The rows comprise Tag and World at three qualities each and seven OG-MARL MuJoCo settings. The paired statistic covers only the 13 settings.

Figure~\ref{fig:dataset-quality-gain} reports $100(R_{\mathrm{online}}-R_{\mathrm{pre}})/R_{\mathrm{pre}}$, where $R_{\mathrm{online}}$ and $R_{\mathrm{pre}}$ are raw deployment returns after online fine-tuning and offline pretraining, respectively. Every included pretrained return is positive. The percentages describe individual trained policies and use a different evaluation protocol from the aggregates in the main tables. Values are rounded to whole percentages only for display. MPE, MA-MuJoCo, and SMACv2 use 50M online steps, including the historical 50M endpoints for 2halfcheetah rather than the later 100M results. SMAC uses the existing budgets for each task: 15M for 3m and 2s3z, 50M for 5m6m, and 31.122M for 8m. Budgets are matched across qualities within each task, but not across all tasks. Missing quality levels are not interpolated. SMACv2 includes only Replay and is not used to infer an effect of dataset quality.

\paragraph{Comparing offline training, online training, and online fine-tuning.}
Figure~\ref{fig:main-training-curves} compares three training configurations across 46 combinations of tasks and datasets. Spread Medium-Replay and 2halfcheetah Poor complete 50M online steps. The 4ant Poor and 8m Poor curves show the available evaluations toward the same budget. The offline segment evaluates our pretrained student during 1M updates of joint flow learning, distillation with value guidance, and value learning \citep{macflow}. The final offline checkpoint supplies the \emph{MAC-Flow reference trained on offline data}, before any online fine-tuning. The offline results are local reproductions under the stated evaluation protocol. The colored online curves show \method{}, initialized from the same pretrained students and improved through online fine-tuning. Gains over the offline endpoint quantify the benefit of online fine-tuning.

The dashed curves supply the complementary \emph{MAPPO reference trained from random initialization}: the actor with the same architecture and centralized critic are trained from random initialization, without offline data or reference KL. The random initialization baseline uses the MA-FPPO policy construction and online optimization algorithm. Comparing training from random initialization with online fine-tuning tests the practical benefit of the complete pretraining and online fine-tuning procedure. When the full configuration includes reference KL regularization, the comparison reflects both initialization and regularization. The reference KL ablation in Figure~\ref{fig:component-ablation}(c) separately examines regularization. The comparison of BC initialization remains in the ablation in the main text, Figure~\ref{fig:parameterization-combined}(b): both flow and BC initializations use the same online optimization algorithm, testing the choice of offline initialization.

The figure thus connects two questions: how much online fine-tuning improves an offline MAC-Flow policy, and how the resulting policy compares with learning online from scratch. The observed improvement over offline initialization supports the first claim. Performance relative to MAPPO trained from random initialization varies across tasks and datasets at matched online steps. Comparisons of online sample efficiency count environment interactions and exclude the cost of collecting the offline data and performing pretraining.

\paragraph{Training budgets.}
The offline and online segments have separate horizontal scales and occupy equal visual widths. Segment lengths do not compare computational cost. The original \method{} budgets are 50M online transitions for MPE, 2ant, 4ant, 5m6m, and SMACv2, 100M for 2halfcheetah, 15M for 3m and 2s3z, and 31.122M for 8m, where training was stopped early. MAPPO trained from random initialization completes 50M online transitions on all 13 tasks. Comparisons across methods use the shared observed budget. SMACv2 includes Replay only, and the Terran and Zerg 5v5 offline zero points use the recorded initialization reconstruction.

All curves show raw team returns. Return gains and remaining gaps depend on task and dataset quality: MPE initialization from lower quality data can catch up with initialization from higher quality data, while several SMAC tasks approach similar plateaus at high returns after online fine-tuning. Returns on 3-Hopper fluctuate substantially during both offline pretraining and online fine-tuning, and the precise cause remains unclear.

\begin{figure}[!htb]
\centering
\includegraphics[width=\linewidth]{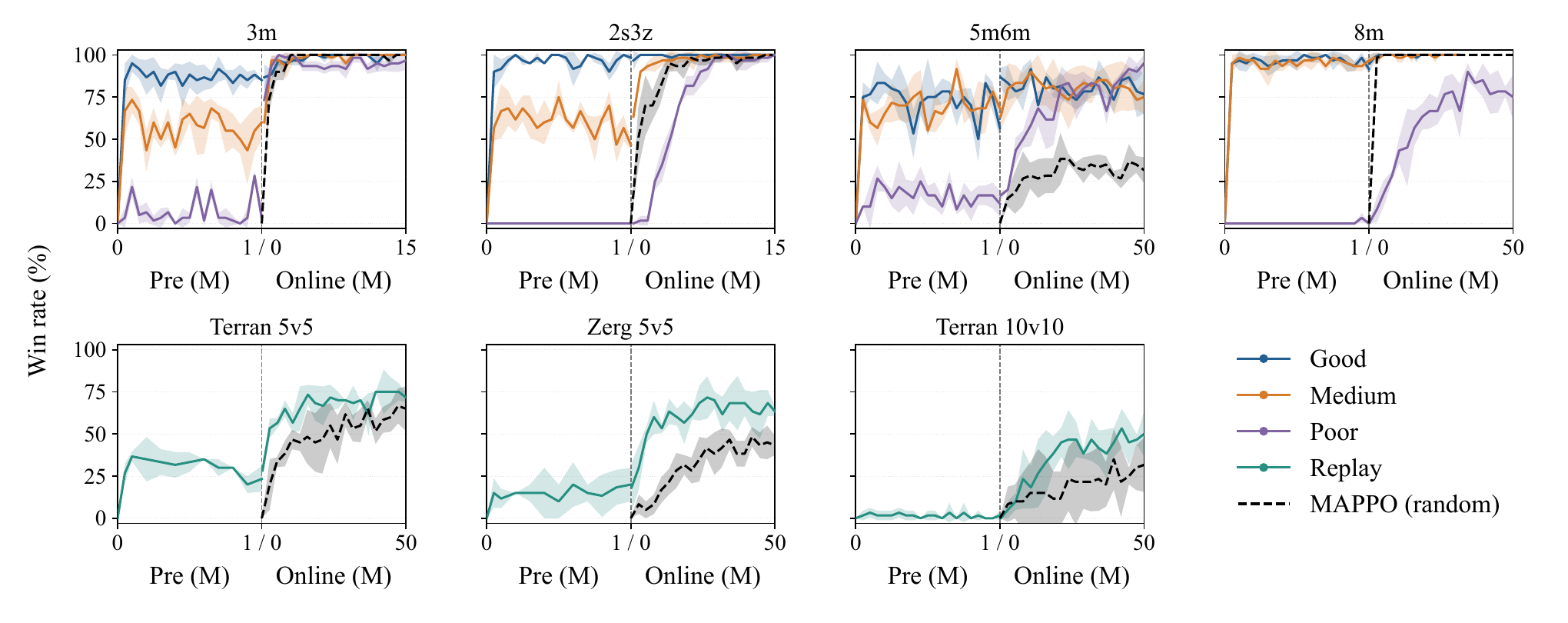}
\caption{\textbf{Deployment win rates on SMAC and SMACv2.} Available SMAC Good/Medium/Poor settings and SMACv2 Replay tasks. Vertical dashed lines separate offline updates and online transitions, both measured in millions. Black dashed curves show MAPPO with the same architecture and random initialization.}
\label{fig:main-smac-winrate}
\end{figure}

Figure~\ref{fig:main-smac-winrate} complements the return curves with the fraction of evaluation episodes won. The 3m and 2s3z settings approach high win rates after online fine-tuning, including initialization from lower quality data. The 8m Good and Medium curves are near the observed ceiling, whereas 5m6m and the SMACv2 tasks retain more visible variation during online fine-tuning. Win rate measures completed battles and can therefore distinguish task success from improvements in shaped return. The reconstruction of the pretrained SMACv2 starting point and the differing online budgets follow the conventions of Figure~\ref{fig:main-training-curves}.

\paragraph{Direct flow policy fine-tuning.}
Figure~\ref{fig:direct-flow-policy-finetuning} compares \method{} with Direct Flow Policy Fine-tuning, which directly updates the pretrained student generator while retaining random latent inputs. The comparison uses the generative MAPPO implementation described in Appendix~\ref{app:zero-step-ablation}. All curves cover the same 50M online budget.

\begin{figure}[!htbp]
\centering
\includegraphics[width=\linewidth]{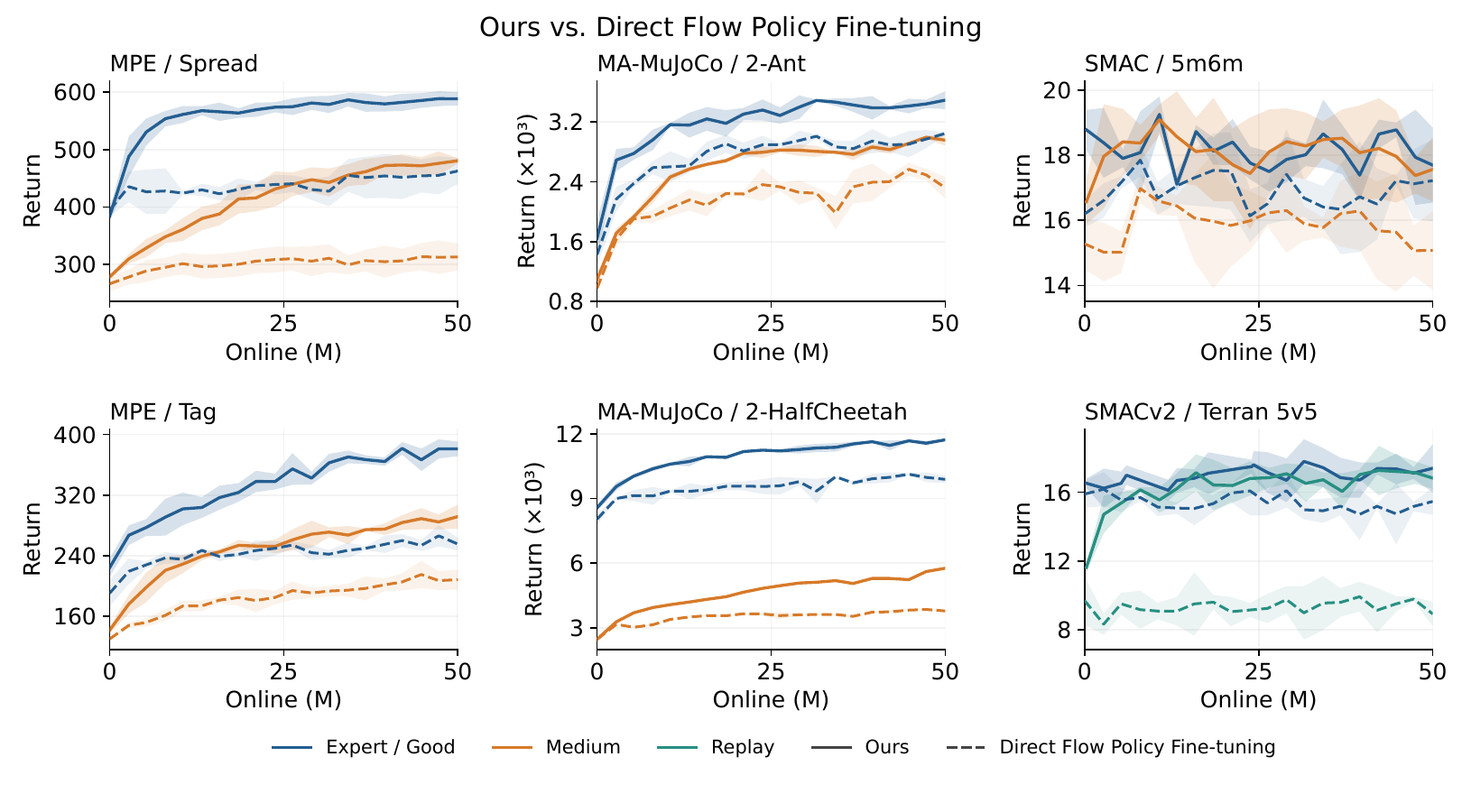}
\caption{\textbf{Comparison with direct flow policy fine-tuning.} Solid curves show \method{} and dashed curves show Direct Flow Policy Fine-tuning. Matching colors identify the same offline data quality. Horizontal axes show online environment steps in millions.}
\label{fig:direct-flow-policy-finetuning}
\end{figure}

\subsubsection{Rationale for Fixing the Latent during Online Fine-Tuning}
\label{app:exploration-rationale}
Fixing the latent retains a useful pretrained action mapping while making online exploration directly controllable. Flow pretraining learns from the behavior represented in the offline data, whereas online fine-tuning optimizes team return through further interaction. The distribution used to learn behavior and the distribution used for online exploration therefore serve different purposes. We explain the choice through evidence from motor learning, the properties of our policy construction, and the experimental comparisons.

\paragraph{Evidence from motor learning.}
\citet{dhawale2019adaptive} trained rats to press a joystick for rewards determined by movement direction. To distinguish the influence of reward from movement accuracy, the authors inserted trials and blocks in which reward was assigned independently of direction. Reward delivery reduced subsequent movement variability, whereas withholding reward increased it. Keeping the target fixed over extended training reduced variability regulation, and returning to changing targets restored it. The analyses identified a fast response to recent rewards and a slower adjustment to task uncertainty. Their policy gradient simulations used a Gaussian policy with an adaptive mean, regulated exploratory variance, and a separate source of motor noise. Appropriate regulation improved learning and reward accumulation. Increasing the regulation gain above its optimum accelerated learning but reduced accumulated reward. The animal experiments established that reward can causally regulate movement variability, while the simulations explained how regulating exploration can support both learning and performance. The relevant principle is adaptive control of exploration, rather than either eliminating variation or maximizing it.

\paragraph{What fixing the latent changes.}
During continuous online fine-tuning, the student output $g_\theta(h,0)$ is the Gaussian mean in Eq.~\eqref{eq:gaussian}, where $h$ denotes an agent's local input. Zero is the latent used by the pretrained student's deterministic deployment policy, so retaining it preserves that initial action mapping. The student parameters continue to change with team advantages, and different observations can produce different actions. Fixing the latent consequently preserves an initialization without freezing the policy. It also differs from averaging the generator's outputs over latent samples, which need not equal $g_\theta(h,0)$.

For fixed $h$ and parameters, let $Z$ follow the generator's base distribution and assume that $g_\theta(h,Z)$ has finite second moments. Define the covariance induced by latent sampling as $C_\theta(h)=\operatorname{Cov}_Z[g_\theta(h,Z)]$. To distinguish the sources of variation, consider raw continuous actions with independent Gaussian noise $\eta\sim\Normal(0,\Sigma)$:
\begin{align*}
 U_Z&=g_\theta(h,Z)+\eta,
 &U_0&=g_\theta(h,0)+\eta,\\
 \operatorname{Cov}(U_Z\mid h)&=C_\theta(h)+\Sigma,
 &\operatorname{Cov}(U_0\mid h)&=\Sigma.
\end{align*}
Independence makes the cross covariance zero, giving the first identity, while the fixed generator output gives the second. With the same $\Sigma$, fixing the latent removes a positive semidefinite component of raw action covariance. In \method{}, $\Sigma=\operatorname{diag}(\sigma_\theta^2)$ is learned directly. This calculation holds the additive noise covariance fixed and concerns raw actions before clipping. The implemented policies can learn different noise scales. The action means generally differ, and returns additionally depend on the environment and subsequent learning, so the covariance identity alone does not rank the policies by return.

When the latent is sampled, generator updates can change both the action mean and the variation induced by the latent. Fixing the latent removes that additional source of variation, leaving the exploration scale under separate control through $\sigma_\theta$. The Gaussian policy also provides the explicit raw action likelihood used by PPO. Direct Flow Policy Fine-tuning learns transition noise as well, so the distinction concerns the additional variation from the base latent and the resulting policy construction. The analogy with motor learning is the separation of learned behavior from adjustable exploration. Our optimizer does not implement the biological reward regulation rule, and the neuroscience study did not compare fixed and random latent policies.

\paragraph{Evidence from our experiments.}
Figure~\ref{fig:direct-flow-policy-finetuning} compares online policy constructions after flow pretraining. Direct Flow Policy Fine-tuning retains the random base latent and updates the student generator, whereas \method{} uses a fixed latent with an explicit action policy. At 50M online steps, \method{} achieves higher return in all 12 displayed combinations across Spread, Tag, 2ant, 2halfcheetah, 5m6m, and Terran 5v5. Eight comparisons concern continuous control with Gaussian policies, and four concern discrete control with categorical policies. The continuous results directly support the Gaussian construction, while the discrete results show that the benefit extends to the corresponding categorical construction.

For example, final returns are 480.62 versus 312.84 on Spread Medium and 5766.14 versus 3787.57 on 2halfcheetah Medium, with \method{} listed first. The gap is smaller on 5m6m Good, at 17.69 versus 17.22. The advantage therefore varies with the task and data quality. These results support the complete online policy construction at the tested budget. Because the constructions also use different deployment rules, the return comparison reflects the combined effects of learning and execution.

The BC comparison in Figure~\ref{fig:ablation-endpoints}(a) addresses a complementary question: whether flow pretraining remains useful after fixing the latent. With Gaussian online fine-tuning in both cases, flow pretraining produces higher final returns than BC pretraining in all four tested combinations. Thus, retaining the full random generator during online fine-tuning is not required to obtain an empirical benefit from the flow pretraining procedure. Figure~\ref{fig:main-training-curves} additionally compares the complete procedure with MAPPO trained from random initialization. That comparison assesses the overall benefit of pretraining and the configured regularization, rather than isolating the latent choice.

\paragraph{Interpretation of the combined evidence.}
The mathematical identity identifies the variation removed by fixing the latent. The biological evidence motivates adapting exploration to task demands, and the experiments support using the pretrained action mapping with an explicit online policy. We hypothesize that removing latent sampling makes the relationship between observations, actions, and team outcomes easier to refine while the learned standard deviation retains useful exploration. When independent latent samples select different behavior patterns across agents or successive decisions, they can introduce variation that complicates coordination. Fixing the latent removes this source, but shared team advantages and interaction are still needed to learn coordinated actions. The existing comparisons do not separate the contributions of latent fixing, likelihood computation, exploration scale, and deployment. They establish an empirical advantage of the tested construction, while the exact cause of the return improvement remains unresolved.

\subsection{Ablation Studies}
\label{app:zero-step-ablation}
\label{app:component-ablation}
\begin{figure}[!t]
\centering
\includegraphics[width=\linewidth]{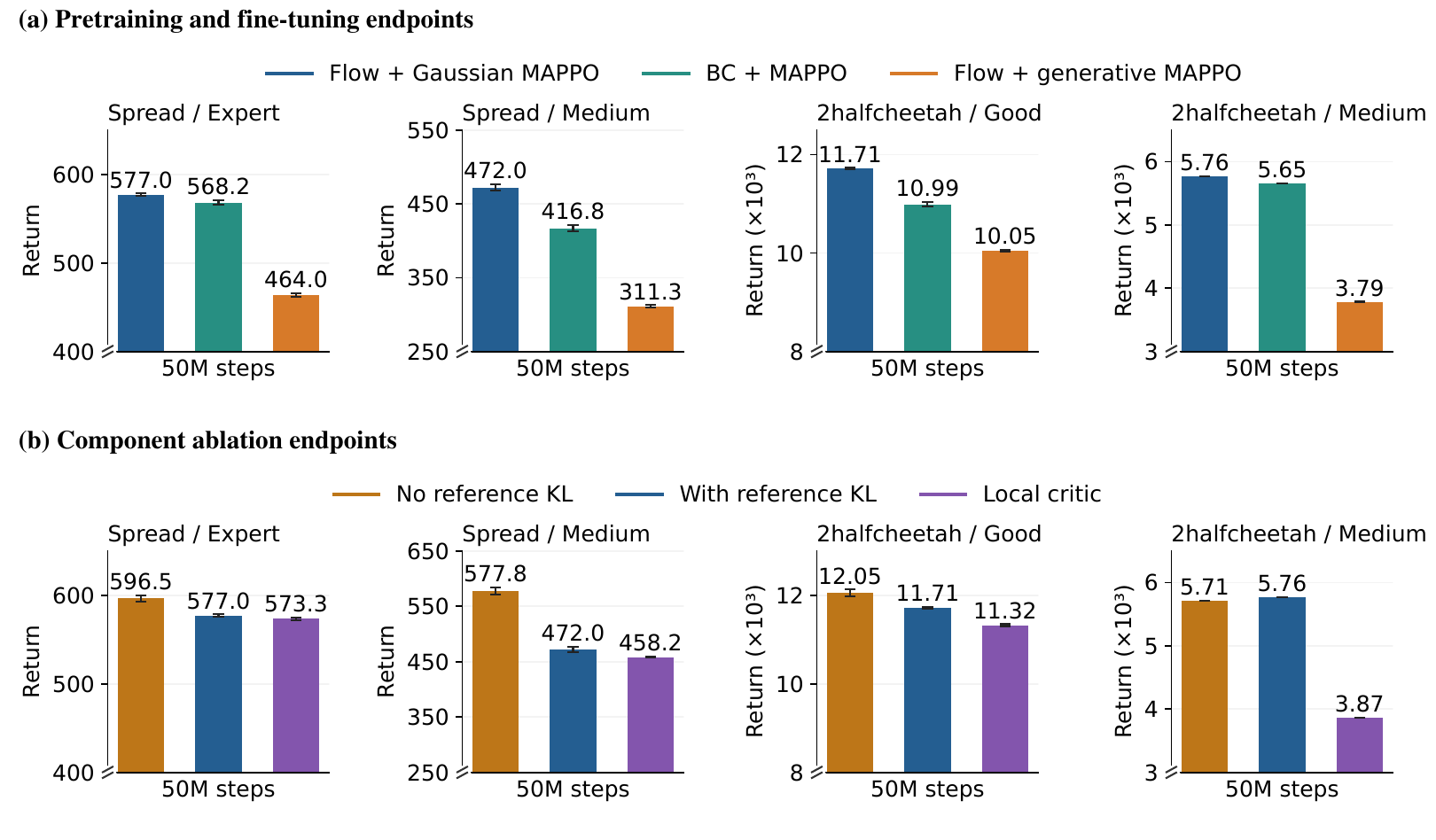}
\caption{\textbf{Endpoint comparisons at 50M online steps.} (a) Gaussian MAPPO initialized from flow pretraining, Gaussian MAPPO initialized from BC, and generative MAPPO initialized from flow pretraining. (b) No reference KL, the main configuration with reference KL, and local critic. Columns show Spread Expert/Medium and 2halfcheetah Good/Medium. Vertical axes are truncated. Corresponding curves appear in Figure~\ref{fig:parameterization-combined}(b,c).}
\label{fig:ablation-endpoints}
\end{figure}

Figure~\ref{fig:ablation-endpoints}(a) evaluates the same 50M checkpoints as the curves in Figure~\ref{fig:parameterization-combined}(b), on Spread Expert/Medium and 2halfcheetah Good/Medium. Deployment with a fixed latent uses the clipped mean at $z=0$. Generative MAPPO samples the base latent with additional transition noise disabled. All three methods use the same environment budget. Deployment follows each policy definition. Gaussian MAPPO uses raw action likelihoods at a fixed latent, while generative MAPPO conditions generation on a sampled latent and computes the likelihoods of the sampled transitions. The truncated vertical axes are marked explicitly.

The HalfCheetah comparison is limited to 50M online steps. 

The BC curves in Figure~\ref{fig:parameterization-combined}(b) cover Spread Expert/Medium and 2halfcheetah Good/Medium. The BC policies use direct BC initialization followed by Gaussian MAPPO. Each column corresponds to one combination of task and dataset. Bars are ordered Gaussian MAPPO initialized from flow pretraining, Gaussian MAPPO initialized from BC, and generative MAPPO. Matching blue, green, and orange colors identify methods in the curves and bars. Tag and 2ant are not included in this figure.

Our choice of flow matching is motivated primarily by performance after online fine-tuning. With the Gaussian policy, online optimization algorithm, and 50M online budget held fixed, flow pretraining yields higher final mean return than BC pretraining in all four settings in Figure~\ref{fig:ablation-endpoints}(a): 577.02 versus 568.19 on Spread Expert, 471.99 versus 416.78 on Spread Medium, 11712.77 versus 10989.10 on HalfCheetah Good, and 5764.36 versus 5654.21 on HalfCheetah Medium. Thus, the benefit of the chosen initialization remains observable after substantial online fine-tuning. The four comparisons provide the direct empirical reason for using the flow pretraining procedure in \method{}.

A plausible explanation is that the offline pretraining procedure combines learning the behavior distribution with policy improvement guided by value estimates. For continuous actions, BC with squared error targets the conditional mean of the dataset actions. When distinct successful behaviors require different actions, the mean action can be less effective than either behavior. A conditional flow teacher can represent multiple action modes, while distillation and value guidance train a student toward the teacher's behavior and actions with higher estimated value. A pretrained policy with higher return can then generate more useful team trajectories for subsequent online policy updates under a finite interaction budget. The proposed explanation concerns the quality of the student's deployed action at $z=0$. The teacher represents multiple modes during pretraining.

The BC comparison measures the value of the complete flow pretraining procedure under the same Gaussian online algorithm. The generative MAPPO comparison addresses the complementary choice of online policy after flow pretraining. Higher final returns from flow initialization remain observable despite the Gaussian online distribution. The results support using the pretrained behavior with explicit action likelihoods in the tested settings.

Figure~\ref{fig:ablation-endpoints}(b) compares the deployed policies in the four displayed settings. Each bar shows the return at the 50M checkpoint of the corresponding curve in Figure~\ref{fig:component-ablation}(c). Truncated vertical axes are marked explicitly. The variant with a local critic retains reference KL regularization. The variant without reference KL retains the centralized critic.

\subsection{Scope of the Comparisons}
The comparisons use the budgets and protocols in Appendix~\ref{app:relative-gains}. The flow versus BC comparison evaluates complete pretraining procedures, including distillation and value guidance. Gaussian and generative MAPPO are compared as implemented, without isolating the effect of multimodal action modeling on exploration. Unavailable published results are excluded from aggregates.

\subsection{Recorded Task Trajectories}
\label{app:recorded-trajectories}
Figures~\ref{fig:trajectories-continuous} and~\ref{fig:trajectories-discrete} show recorded trajectories for continuous and discrete tasks, respectively. Each row follows one episode through six selected states.

The two figures serve different purposes. The continuous trajectories show behavior present in the offline datasets and provide context for the actions learned during pretraining. The discrete trajectories show successful execution by fine-tuned policies. Reading them together connects the source of pretrained behavior with its use in cooperative control, rather than presenting the two figures as a before-and-after comparison of the same policy.

For MPE, the position trails show how agents move relative to teammates, landmarks, and targets. Landmark coverage and capture contact help relate the movements to the task objective. For MuJoCo, each view is centered on the robot, so forward progress should be read from the reported velocity together with the sequence of body configurations. Colors identify the actuator groups controlled by different agents, making their contributions to the shared motion visible.

For the discrete tasks, unit positions, health, and survivor counts show how the engagement develops over time. Attack links help interpret which opponents the allied units target. The six states summarize selected moments; the corresponding supplementary videos show the intervening motion. These examples complement the return curves by illustrating behavior within individual episodes. Aggregate performance comparisons remain based on the quantitative results.
\begin{figure}[p]
\centering
\includegraphics[width=\linewidth,height=.875\textheight,keepaspectratio]{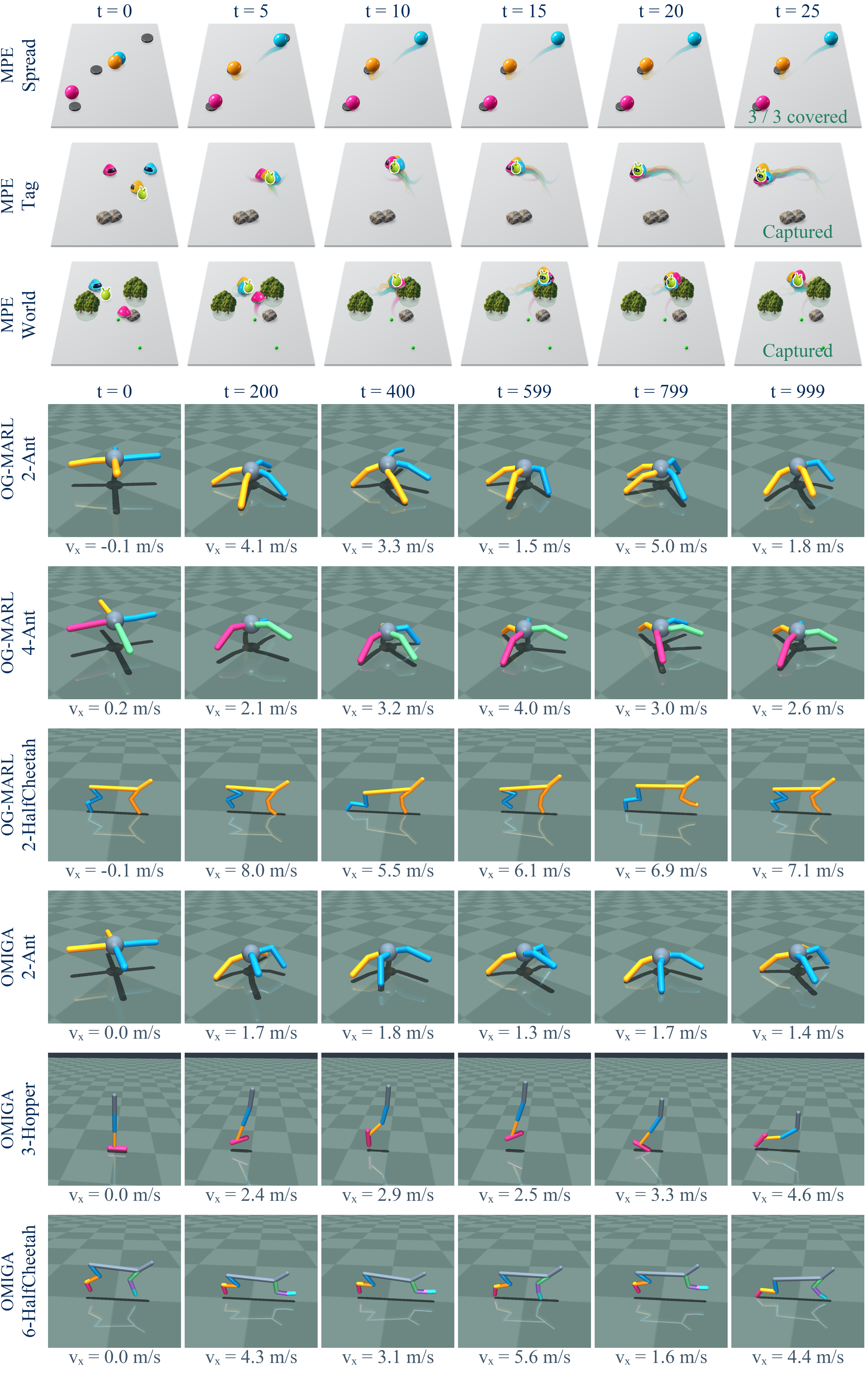}
\caption{Recorded trajectories in tasks with continuous actions. Each row shows six states from an episode with high return: MPE Expert, OG-MARL Good, or OMIGA Expert. Time headers are shared within MPE and MuJoCo. MPE trails follow recorded positions. Green labels mark landmark coverage or capture contact. MuJoCo poses use recorded joint configurations and forward kinematics, with the body centered in each view. Both OG-MARL and OMIGA report the recorded forward velocity $v_x$ in m/s. Colors distinguish actuator groups. Task prefixes count agents controlling one robot.}
\label{fig:trajectories-continuous}
\end{figure}
\begin{figure}[p]
\centering
\includegraphics[width=\linewidth,height=.858\textheight,keepaspectratio]{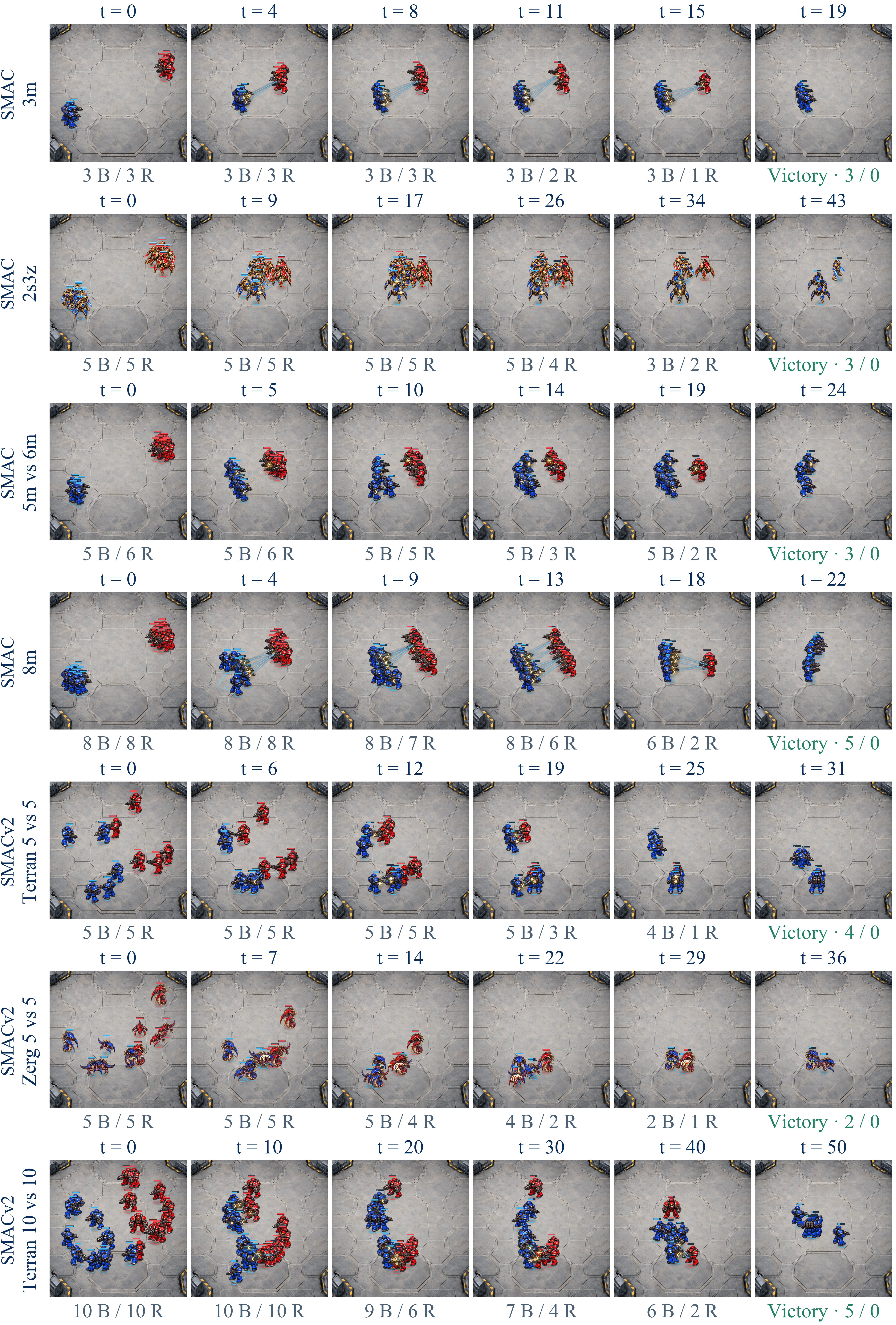}
\caption{Recorded trajectories in tasks with discrete actions. Each row shows six states from a successful episode from a fine-tuned MA-FPPO policy, ending with no surviving enemies. Blue and red indicate allied and enemy units. Counts report survivors. Positions, health, and disappearance follow the records. Sprite facing uses recorded allied action targets when available, otherwise displacement or the previous direction. Zerg body motion, hydralisk spines, and zergling claw strokes illustrate recorded allied attack commands. In 5m vs 6m, missing attack targets are illustrated by dashed links inferred from damage and proximity. Other links follow recorded commands. The selected successful episodes illustrate behavior rather than average performance.}
\label{fig:trajectories-discrete}
\end{figure}

\clearpage

\section{Method Details and Theoretical Foundations}
\label{app:method-details}
\subsection{Notation and Index Conventions}
\label{app:notation}
Environment time $t$, flow time $\tau$, Euler step $k$, and optimizer update index $n$ describe different operations. A teacher integrates over $K=10$ flow steps for one target. A deployed student uses one forward pass for one agent decision. Online budgets count joint environment transitions across parallel environments. Subscripts such as $h_i$ and $u_i$ suppress environment time when discussing a single decision. The abbreviated symbols correspond to $h_t^i$ and $u_t^i$ in rollout equations. Bold symbols collect all controlled agents. Superscripts $\mathrm{off}$, $\mathrm{raw}$, $\mathrm{disc}$, and $\mathrm{cont}$ identify quantities or branches. The superscripts are labels rather than powers. All logarithms are natural.

\small
\begin{longtable}{@{}p{.25\linewidth}p{.70\linewidth}@{}}
\caption{Symbols used in training, deployment, and analysis.}\label{tab:notation}\\
\toprule Symbol & Meaning \\\midrule
\endfirsthead
\multicolumn{2}{l}{\textit{Table~\ref{tab:notation} continued}}\\
\toprule Symbol & Meaning \\\midrule
\endhead
\bottomrule\endfoot
\multicolumn{2}{l}{\textbf{Indices, clocks, and budgets}}\\
$i,N$ & Index of a controlled agent and number of controlled agents.\\
$t$ & Environment time along an episode or rollout. Recurrent states advance at environment steps.\\
$\tau\in[0,1]$ & Continuous flow interpolation time: 0 is latent noise and 1 is the data endpoint.\\
$k,K,\tau_k$ & Euler integration index, number of integration steps, and grid time $\tau_k=k/K$. $K=10$. Euler steps are generation steps within one teacher query.\\
$n,U$ & Offline optimizer update index and total offline updates. The reported runs use $U=10^6$.\\
$C,B_{\mathrm{env}}$ & Number of collected joint transitions and online transition budget in Algorithm~\ref{alg:ma-fppo}.\\
$j,\mathcal B$ & Offline minibatch sample index and minibatch. $|\mathcal B|$ is the batch size.\\
$l$ & Index of future offsets in the GAE sum. The product index $j$ in Eq.~\eqref{eq:gae} is a local dummy index for time offsets.\\
$b,e_b$ & Discrete action index and corresponding one-hot encoding, distinct from the bootstrap mask $b_t$.\\
\multicolumn{2}{l}{\textbf{Environment and offline generation}}\\
$\D$ & Offline trajectory dataset.\\
$o_t^i,h_t^i,s_t$ & Local observation, actor input including agent identity, and centralized critic input. $h_t^i$ includes recurrent features where used. $s_{t+1}^{+}$ is the successor critic input before reset.\\
$a_t^i,\mathbf a_t,r_t$ & Executed local action, joint action, and shared team reward.\\
$d_a$ & Continuous action dimension or number of one-hot coordinates for a discrete action.\\
$z_i,I$ & Standard Gaussian latent and identity covariance matrix. Online students use $z_i=0$.\\
$x_1^i,x_\tau^i$ & Dataset action vector and interpolated flow input $(1-\tau)z_i+\tau x_1^i$. The subscript 1 is a flow endpoint.\\
$\widehat x^{i,(k)}$ & Numerically integrated teacher state at Euler step $k$. Parentheses mark an integration index.\\
$v_\phi,F_\phi,g_\theta$ & Teacher velocity field, integrated teacher target, and student output from one forward pass.\\
$\phi,\theta,\omega,\xi$ & Parameters of the teacher head, student/actor, offline Q functions, and shared encoder. Online $\theta$ also includes the Gaussian log standard deviation where used.\\
$\omega^\dagger,\sg[\cdot]$ & Detached Q weights and stop-gradient. Action derivatives through Q remain active for continuous guidance.\\
$\bar Q_\omega,G_j$ & Offline action value averaged over the ensemble and continuous Q value averaged across agents for minibatch sample $j$.\\
$p_\theta^{\mathrm{off}},y_i$ & Offline categorical student probabilities and teacher argmax label.\\
$\mathcal L_Q,\mathcal L_{\mathrm{FM}}$ & Offline temporal-difference and flow matching losses.\\
$\mathcal L_{\mathrm{distill}},\mathcal L_{\mathrm{guide}},\mathcal L_{\mathrm{off}}$ & Distillation, value guidance, and combined offline losses.\\
$\alpha$ & Distillation coefficient, set to 3 for discrete tasks and 1 for continuous tasks.\\
\multicolumn{2}{l}{\textbf{Online policies, advantages, and updates}}\\
$\ell_\theta,\mathcal A_i(h_i),T$ & Student logits, set of legal actions, and categorical temperature, with $T=1$.\\
$\pi_\theta,\Pi_\theta$ & Local distribution of executed actions and the conditional joint product. For continuous control, clipping induces $\pi_\theta$ from $q_\theta$.\\
$q_\theta,\sigma_\theta$ & Raw Gaussian sampling density and vector of standard deviations, shared across states and agents.\\
$u_t^i$ & Raw continuous action before clipping. $a_t^i=\clip(u_t^i,-1,1)$. The raw action is distinct from flow time $\tau$.\\
$\mu_\theta,w_t^i$ & Unified sampling policy and stored sample: $(\pi_\theta,a_t^i)$ for discrete control, $(q_\theta,u_t^i)$ for continuous control.\\
$\theta_{\mathrm{old}},\psi_{\mathrm{old}}$ & Actor and critic parameters saved for rollout collection and target construction, refreshed between rollouts.\\
$\theta_{\mathrm{ref}},\pi_{\mathrm{ref}},q_{\mathrm{ref}}$ & Frozen initial online actor and corresponding discrete or raw Gaussian reference policy, fixed across all online updates.\\
$V_\psi,\psi$ & Online centralized state-value function and corresponding parameters, separate from offline action-value parameters $\omega$.\\
$\gamma,\lambda$ & Return discount and GAE trace coefficient.\\
$b_t,c_t$ & Bootstrap and trace continuation masks. True terminals stop bootstrap. Resets stop traces. Time limits use final states before reset for bootstrap.\\
$\delta_t$ & Temporal-difference residual using the critic at collection.\\
$\widehat A_t^{\mathrm{raw}},\overline A,\widetilde A_t$ & Raw GAE, rollout mean GAE, and normalized actor advantage shared across agents. Normalization uses the population standard deviation of rollout samples.\\
$\widehat R_t$ & Detached value target formed before advantage normalization.\\
$m_t^i,Z$ & Actor activity mask and count normalization $\max(1,\sum_{t,i}m_t^i)$. Continuous agents are active. Discrete agents with only one legal action are excluded.\\
$\rho_t^i,\epsilon$ & Ratio of sample likelihoods under the current and collection policies and PPO clipping parameter. $\widehat D_{\mathrm{old}}$ is the discrete sampled KL estimate for stopping.\\
$L_{\mathrm{clip}},\mathcal L_{\mathrm{actor}}$ & Maximized PPO surrogate and minimized full actor loss.\\
$\mathcal R_{\mathrm{ref}},\beta$ & Reference KL regularization and coefficient. Discrete KL is from reference to current. Continuous KL is from current to reference.\\
$\mathcal H,\eta$ & Entropy of the sampling policy and entropy coefficient. Gaussian entropy is differential entropy of the raw sample.\\
$\E^m_{t,i},\E,\KL,\clip$ & Average weighted by activity, expectation, KL divergence, and clipping operator.\\
\multicolumn{2}{l}{\textbf{Theoretical analysis and reported scores}}\\
$\chi,M_\theta,d_\theta$ & Markov information state, joint sampling policy, and normalized discounted occupancy in the analysis. The flow input is $x_\tau^i$.\\
$J,A^\theta,A^{\mathrm{old}}$ & Expected discounted team return, exact team advantage, and advantage of the collection policy. Estimated GAE uses a hat.\\
$f,f_\#q$ & Action map independent of the parameters and induced distribution after applying the action map. Continuous execution uses clipping.\\
$\mathsf P_1,\mathsf P_2,\mathsf P_j^a$ & Probability measures of raw actions and corresponding marginals after execution in the clipping proof. $\mathsf P_j(\cdot\mid a)$ denotes conditioning on the executed action.\\
$\Pi,\widetilde\Pi,\pi_i,p_i$ & Generic joint and local product distributions in the proof of KL factorization.\\
$S_{\mathrm{agent}},S_{\mathrm{joint}}$ & Centered surrogates using individual and joint likelihood ratios.\\
$M,M',P_M$ & Old/new joint sampling policies and the environment transition kernel induced by $M$. A prime denotes a second policy.\\
$A_*,D_{\mathrm{TV}},\varepsilon_{\mathrm{TV}}$ & Uniform absolute advantage bound, total variation distance, and maximum policy total variation distance. The teacher velocity field is $v_\phi$.\\
$\mathcal G(M',M),\zeta(\chi)$ & Surrogate gain for the joint policy and expected advantage at each state in the proof. Offline value guidance uses $G_j$. The student network is $g_\theta$.\\
$\delta_i,\delta_{\mathrm{joint}}$ & Uniform individual and joint KL bounds, distinct from the temporal-difference residual $\delta_t$. $\delta$ bounds KL averaged across agents at each history.\\
$C_{\mathrm{clip}},\Delta$ & Unnormalized ideal clipped surrogate and residual between the joint surrogate and the sum of individual surrogates.\\
$\mathcal C,c,S_c,B_c,\overline G_{\mathcal C}$ & Completed settings, a setting index, online score, strongest listed offline score, and average relative gain. $B_c$ is a score, while $B_{\mathrm{env}}$ is a transition budget.\\
$R,R_{\mathrm{random}},R_{\mathrm{expert}}$ & Mean episode return and MPE reference returns for normalization. The score quantities differ from value target $\widehat R_t$.\\
$R_{\mathrm{pre}},R_{\mathrm{online}}$ & Pretrained and online endpoint returns used in the heatmap of gains across data qualities.\\
$\Pi_F,\Pi_B,\widehat Q,\varepsilon_Q,m$ & Deployed flow and BC policies, estimated BC team action value, uniform estimation error bound, and estimated value margin.\\
$\mathcal R_{\mathrm{ret}}$ & Return functional for a single decision in the proof of the policy gradient for raw actions.\\
\end{longtable}
\normalsize

\subsection{Offline Objectives and Network Architectures}
\label{app:offline}
Teacher integration uses $K=10$ Euler steps indexed by $k$, with $\tau_k=k/K$:
\begin{equation}
 \widehat x^{i,(0)}=z_i,\qquad
 \widehat x^{i,(k+1)}=\widehat x^{i,(k)}+
 K^{-1}v_\phi(h_i,\widehat x^{i,(k)},\tau_k),\quad k=0,\ldots,K-1.
 \label{eq:euler}
\end{equation}
The teacher target $F_\phi(h_i,z_i)$ is the final $\widehat x^{i,(K)}$, clipped to $[-1,1]$ for continuous control. A student that predicts actions in one forward pass $g_\theta(h_i,z_i)$ learns from the teacher target and from offline value estimates.

\paragraph{Student distillation.}
The student learns to reproduce the output of teacher integration in a single forward pass. Teacher and student receive the same local information and latent, so distillation matches teacher and student outputs for the same input. Let $\omega$ denote offline Q parameters and $\xi$ any shared encoder parameters. The teacher and student heads have parameters $\phi$ and $\theta$. $h_i$ depends on $\xi$ when a shared encoder is used. Write $\sg[\cdot]$ for stop-gradient, $p_\theta^{\mathrm{off}}(\cdot\mid h_i,z_i)=\operatorname{softmax}(g_\theta(h_i,z_i))$, and $\bar Q_\omega$ for the Q value averaged over the ensemble. The distillation terms are
\begin{align}
 \mathcal L_{\mathrm{distill}}^{\mathrm{disc}}
 &=-\E_{\D,z,i}\log p_\theta^{\mathrm{off}}(y_i\mid h_i,z_i),
 &y_i&=\arg\max_b\sg[F_\phi(h_i,z_i)]_b,\label{eq:distill-disc}\\
 \mathcal L_{\mathrm{distill}}^{\mathrm{cont}}
 &=\E_{\D,z,i}\frac{\|g_\theta(h_i,z_i)-\sg[F_\phi(h_i,z_i)]\|_2^2}{d_a},
 &&\label{eq:distill-cont}
\end{align}
where $y_i$ is the action label selected by the teacher, $b$ indexes discrete actions, and $F_\phi$ includes clipping to $[-1,1]$ in continuous control. Distillation uses the same noise for the teacher target and the corresponding student prediction. \paragraph{Student learning with value guidance.}
Distillation encourages agreement with the teacher, whereas value guidance encourages actions with higher estimated return. Distillation and value guidance train the student to follow the teacher while favoring actions with higher offline Q estimates. For discrete actions, the student raises the probability of actions with high Q values by minimizing the negative expected Q value. The Q values are detached:
\begin{equation}
 \mathcal L_{\mathrm{guide}}^{\mathrm{disc}}
 =-\E_{\D,z}\sum_i\sum_b
 p_\theta^{\mathrm{off}}(b\mid h_i,z_i)\,\sg[\bar Q_\omega(h_i,e_b)],
 \label{eq:guide-disc}
\end{equation}
where $e_b$ is a one-hot action. For continuous actions, gradients through Q with respect to the proposed action move the student toward outputs with higher Q values. Define $\omega^\dagger=\sg[\omega]$ to freeze Q weights while retaining derivatives with respect to the action argument. For offline minibatch $\mathcal B$, let
\begin{align}
 G_j&=\frac1N\sum_i\bar Q_{\omega^\dagger}
       (h_{j,i},\clip(g_\theta(h_{j,i},z_{j,i}),-1,1)),\nonumber\\
 \mathcal L_{\mathrm{guide}}^{\mathrm{cont}}
 &=-\frac{|\mathcal B|^{-1}\sum_{j\in\mathcal B}G_j}
 {\sg[\max(|\mathcal B|^{-1}\sum_{j\in\mathcal B}|G_j|,10^{-6})]}.
 \label{eq:guide-cont}
\end{align}
In the continuous guidance loss, $j$ indexes minibatch samples, $|\mathcal B|$ is the batch size, and $G_j$ is the mean action value across agents for sample $j$. The normalization scales the continuous guidance term without changing the gradient direction within a minibatch. Q weights are fixed during guidance, but the derivative through the student's action is retained. Value guidance uses an offline action-value function. Online fine-tuning uses a separate centralized critic.

In the discrete implementation, one-hot dataset actions train the local velocity field in Eq.~\eqref{eq:fm}. A local long short-term memory (LSTM) encoder has hidden dimension 64, and the actor and value MLPs have four hidden layers of width 256. The teacher integrates for 10 Euler steps and supplies hard action labels to the student's cross-entropy distillation objective. Two Q estimates are averaged before summing action values across agents. Q targets use the teacher's next action. Value guidance computes the categorical expectation of Q values for each action with Q values detached. The default distillation coefficient is 3. The online LSTM parameters remain frozen, but recurrent state is updated and reset with the actual rollout histories.

The continuous implementation has four hidden layers of 512 units in the teacher, student, and each of two Q networks. Teacher and student receive the local observation and a latent with the same dimension as the action. The teacher also receives flow time. Agent identity is included in the observation representation. Teacher outputs are obtained by 10 Euler steps and clipped to the legal action interval before distillation using squared error. Q values are averaged across the ensemble and controlled agents. The actor's negative Q objective is normalized by a detached mean absolute Q value, with a numerical lower bound of $10^{-6}$. The distillation coefficient is 1. MuJoCo transitions with unavailable information about the next state are excluded from temporal-difference (TD) targets by a critic mask, while the observed actions can still support imitation.

\paragraph{Joint optimization and gradient routing.}
The student distillation and guidance objectives are given in Eqs.~\eqref{eq:distill-disc}--\eqref{eq:guide-cont}, where $\omega$ denotes offline Q parameters and $\xi$ shared encoder parameters. Local features $h_i$ can depend on $\xi$. Teacher and student heads use $\phi$ and $\theta$, respectively. The continuous guidance objective freezes Q weights while retaining action derivatives, as defined above.

Detaching $G_j$ in the numerator would incorrectly remove the student's gradient from value guidance. Both TD targets and teacher distillation targets are detached. In the discrete TD loss, action values averaged over the ensemble are summed across agents and fitted with one-half mean squared error. Continuous TD learning instead averages action values across agents, then averages the squared TD errors of the two critics. The continuous TD target averages the ensemble of target Q functions and uses a next action from the current student. The discrete target uses the current teacher. Target Q weights receive soft updates.

At every offline update, the implementation sums $\mathcal L_Q+\mathcal L_{\mathrm{FM}}+\alpha\mathcal L_{\mathrm{distill}}+\mathcal L_{\mathrm{guide}}$, calls backward once, and takes one optimizer step. The discrete guidance coefficient is 1. The teacher head receives flow matching gradients, the student head receives distillation and guidance gradients, and the Q heads receive TD gradients. In SMAC, the shared encoder can receive gradients from multiple terms through paths where gradients remain active. The loss terms therefore share encoder parameters. The four loss components are optimized in one offline pretraining stage. Figure~\ref{fig:motivation} schematically depicts the functional components.

Online fine-tuning initializes the actor from the final student, optionally freezes a copy as the reference when $\beta>0$, and introduces a new centralized critic. The teacher and offline Q functions no longer participate. Actor and value losses are optimized separately. Eq.~\eqref{eq:actor} includes PPO, reference KL, and entropy terms.

\paragraph{Conditions for Higher Return after Flow Pretraining.}
We give a sufficient condition under which flow pretraining produces a policy with higher return than BC pretraining. Let $\Pi_F$ and $\Pi_B$ denote the actual deployed joint policies from flow and BC pretraining, including action selection with a fixed latent. Let $\chi$ denote a Markov information state containing joint histories and any required environment state. Let $Q^{\Pi_B}(\chi,\mathbf a)$ be the true team action value of BC deployment, and suppose an estimate $\widehat Q$ satisfies $|\widehat Q-Q^{\Pi_B}|\leq\varepsilon_Q$ for actions of both policies at every information state visited by $\Pi_F$. The constant $\varepsilon_Q$ bounds estimation error. For a uniform margin in estimated values $m$, assume at every visited information state that
\[
 \mathbb E_{\mathbf a\sim\Pi_F}\widehat Q(\chi,\mathbf a)
 -\mathbb E_{\mathbf a\sim\Pi_B}\widehat Q(\chi,\mathbf a)\geq m.
\]
Under the assumptions for discounted returns in Appendix~\ref{app:team-gradient}, the performance-difference identity used in Appendix~\ref{app:performance-bound} gives
\[
 J(\Pi_F)-J(\Pi_B)\geq\frac{m-2\varepsilon_Q}{1-\gamma}.
\]
Indeed, the estimated and true expected values of each policy differ by at most $\varepsilon_Q$, so the true expected BC advantage under $\Pi_F$ is at least $m-2\varepsilon_Q$. Averaging over the normalized discounted occupancy of the flow policy proves the bound. Thus $m>2\varepsilon_Q$ is sufficient for a better initialization. The bound concerns the pretrained policies before online fine-tuning and assumes a uniform value margin and error bound, which are not measured in the experiments. The empirical comparisons support the complete pretraining procedure in the four tested settings.

\subsection{Online Fine-Tuning Details}
\label{app:online-details}
We initialize fresh actor and critic optimizers when transferring the pretrained student. The discrete temperature is $T=1$. Agents with only one legal action are excluded from actor loss normalization. All continuous agents are active. Continuous exploration is initialized with $\sigma_\theta=0.2$, and $\log\sigma_\theta$ is constrained to $[-4,0]$. The buffer stores raw Gaussian samples and corresponding log probabilities before clipping. For SMAC and SMACv2, a pretrained recurrent encoder maps each agent's observation history to the local features $h_t^i$. The encoder weights remain frozen during online fine-tuning, while the hidden state updates with observations and resets at episode boundaries. For MPE and MA-MuJoCo, $h_t^i$ is the current local observation.

A new centralized critic $V_\psi(s_t)$ predicts team return. We store critic parameters at rollout collection as $\psi_{\mathrm{old}}$ when constructing rollout targets. The critic uses the environment state in SMAC and MuJoCo and concatenated observations of controlled agents in MPE. GAE uses a common team reward and critic \citep{gae}:
\begin{equation}
 \delta_t=r_t+\gamma b_tV_{\psi_{\mathrm{old}}}(s_{t+1}^{+})-V_{\psi_{\mathrm{old}}}(s_t),
 \qquad
 \widehat A_t^{\mathrm{raw}}=\sum_{l\geq0}(\gamma\lambda)^l
 \left(\prod_{j=0}^{l-1}c_{t+j}\right)\delta_{t+l},
 \label{eq:gae}
\end{equation}
where $\delta_t$ is the temporal-difference residual, $\widehat A_t^{\mathrm{raw}}$ the unnormalized GAE estimate, $\lambda\in[0,1]$ the trace parameter, and $l$ the offset in future steps. The successor input $s_{t+1}^{+}$ is taken before any reset. The binary mask $b_t$ disables bootstrapping at true terminals and $c_t$ cuts traces at resets. The sum is truncated at the rollout boundary, and an empty product equals 1. At a time limit, $b_t=1$ and $c_t=0$. Before actor normalization, we form detached value targets $\widehat R_t=\sg[\widehat A_t^{\mathrm{raw}}+V_{\psi_{\mathrm{old}}}(s_t)]$. Only actor advantages are normalized over the rollout, $\widetilde A_t=(\widehat A_t^{\mathrm{raw}}-\overline A)/\max(\operatorname{std}(\widehat A^{\mathrm{raw}}),10^{-8})$, and shared across agents. The quantities $\overline A$ and $\operatorname{std}(\widehat A^{\mathrm{raw}})$ are the rollout mean and standard deviation, computed with the number of rollout samples as the variance divisor.

The coefficient $\beta\geq0$ controls optional reference KL regularization. $\beta=0$ removes reference KL while retaining PPO clipping and a stopping rule based on KL from the old policy to the current policy. When reference KL is enabled, the reference has fixed parameters $\theta_{\mathrm{ref}}$, copied once from the initial online actor, including the initial exploration scale. Write $\pi_{\mathrm{ref}}=\pi_{\theta_{\mathrm{ref}}}$ and $q_{\mathrm{ref}}=q_{\theta_{\mathrm{ref}}}$. The implementations use different KL directions:
\begin{equation}
 \mathcal R_{\mathrm{ref}}=
 \begin{cases}
  \E_{t,i}^{m}\!\left[\KL(\pi_{\mathrm{ref}}\|\pi_\theta)\right], & \text{SMAC/SMACv2},\\
  \E_{t,i}\!\left[\KL(q_\theta\|q_{\mathrm{ref}})\right], & \text{continuous control}.
 \end{cases}
 \label{eq:anchor}
\end{equation}
Both policies in each KL use the same stored local input $h_t^i$ and, for discrete actions, the same legal mask. We use $\E^m_{t,i}[f]=Z^{-1}\sum_{t,i}m_t^if_t^i$. $\E_{t,i}$ is the ordinary average for continuous agents. The entropy term is $\mathcal H=\E^m_{t,i}[\mathrm H(\mu_\theta(\cdot\mid h_t^i))]$, using categorical entropy or differential entropy of the raw Gaussian policy. Time sums span the current rollout or minibatch, across parallel environments. The flow teacher and offline Q functions remain fixed. The configuration in the main tables uses $\epsilon=0.05$, $\beta=0.01$, and $\eta=0.001$. The ablation without reference KL uses $\beta=0$. Reference KL regularization penalizes deviation from initialization and is optional when transferring the student to online fine-tuning. Section~\ref{sec:components} shows that the tested endpoint returns favor removing reference KL in three of four settings. For discrete actions, the minibatch estimate $\widehat D_{\mathrm{old}}=\E^m_{t,i}[(\rho_t^i-1)-\log\rho_t^i]$ is checked before each actor update. If the estimate exceeds 0.02, the remaining actor and critic updates for that rollout stop. For continuous actions, analytic Gaussian KL is checked on the minibatch after each actor update. A minibatch value above 0.02 triggers a full rollout check, which is also performed at each epoch end. A full rollout value above 0.02 stops further actor updates while critic updates continue. Both checks measure divergence from the old policy to the current policy and allow the threshold to be exceeded after an update.

\subsection{Training and Deployment Algorithms}
\label{app:algorithms}
Algorithms~\ref{alg:offline}--\ref{alg:inference} summarize offline pretraining, online fine-tuning, and decentralized execution. The student is transferred directly between stages. Loss definitions and implementation details for each task are given above. Online hyperparameters are in Table~\ref{tab:hyper}.

\begin{algorithm}[t]
\caption{MA-FPPO: offline pretraining}
\label{alg:offline}
\begin{algorithmic}[1]
\Require Dataset $\D$, update budget $U$, teacher integration steps $K=10$
\State Initialize teacher $\phi$, student $\theta$, encoder $\xi$, Q networks $\omega$, and target Q networks
\For{$n=1,\ldots,U$}
 \State Sample a minibatch from $\D$ and construct local inputs $h_i$ and action vectors $x_1^i$
 \Statex \hspace{\algorithmicindent}\textbf{Flow matching}
 \State Sample $z_i\sim\Normal(0,I)$, $\tau\sim\mathcal U[0,1]$ and set $x_\tau^i=(1-\tau)z_i+\tau x_1^i$
 \State Compute $\mathcal L_{\mathrm{FM}}$ using Eq.~\eqref{eq:fm}
 \Statex \hspace{\algorithmicindent}\textbf{Student learning with value guidance}
 \State Sample fresh $z_i\sim\Normal(0,I)$ and obtain detached $F_\phi(h_i,z_i)$ using Eq.~\eqref{eq:euler}
 \State Compute $\mathcal L_Q$, $\mathcal L_{\mathrm{distill}}$, and $\mathcal L_{\mathrm{guide}}$ with the objectives for each task
 \State Jointly update trainable parameters using $\mathcal L_{\mathrm{off}}$ defined in Eq.~\eqref{eq:offline}
 \State Update the target Q networks by a weighted average
\EndFor
\State \Return Pretrained student and local encoder
\end{algorithmic}
\end{algorithm}

\begin{algorithm}[t]
\caption{MA-FPPO: online fine-tuning}
\label{alg:ma-fppo}
\begin{algorithmic}[1]
\Require Pretrained student, transition budget $B_{\mathrm{env}}$, reference KL coefficient $\beta\geq0$
\State Construct actor $\theta$ at $z_i=0$, initialize exploration and centralized critic $\psi$
\State Freeze reference $\theta_{\mathrm{ref}}\gets\theta$ if $\beta>0$
\State Set transition counter $C=0$
\While{$C<B_{\mathrm{env}}$}
 \Statex \hspace{\algorithmicindent}\textbf{Collect trajectories}
 \State Store collection parameters $\theta_{\mathrm{old}},\psi_{\mathrm{old}}$ and clear the rollout buffer
 \For{each joint step in the rollout, within the remaining budget}
  \State Construct local inputs $h_t^i$ and centralized input $s_t$
  \State Sample $w_t^i\sim\mu_{\theta_{\mathrm{old}}}(\cdot\mid h_t^i)$ for each agent
  \State Execute $a_t^i=w_t^i$ for discrete actions or $a_t^i=\clip(w_t^i,-1,1)$ for continuous actions
  \State Store inputs, $w_t^i$, old policy statistics, reward, legal/activity and bootstrap/trace masks
  \State Advance/reset local states and increment $C$ by the number of environments advanced
 \EndFor
 \State Compute GAE advantages and value targets with $\psi_{\mathrm{old}}$ using Eq.~\eqref{eq:gae}
 \Statex \hspace{\algorithmicindent}\textbf{Optimize actor and critic}
 \For{each optimization epoch and minibatch}
  \State Update $\psi$ by regression to the value targets
  \If{critic warmup is complete and the KL stopping threshold has not been exceeded}
   \State Apply the KL checks in Appendix~\ref{app:online-details} and update $\theta$ using Eq.~\eqref{eq:actor} while permitted
   \State On exceeding the threshold, stop further actor updates for this rollout, and also the remaining critic updates for discrete actions
  \EndIf
 \EndFor
\EndWhile
\State \Return Online actor and local encoder
\end{algorithmic}
\end{algorithm}
The sampling policy $\mu_\theta=\pi_\theta$ is the legally masked categorical policy for discrete actions, and $\mu_\theta=q_\theta$ is the Gaussian policy for continuous actions. The buffer retains raw Gaussian samples before clipping and bootstrap values from final states before reset. Fresh optimizers, the frozen recurrent encoder used in SMAC and SMACv2, exploration initialization and bounds, critic warmup, checks of old log probabilities during replay, and the KL stopping rule within each rollout follow the implementation conventions above and Table~\ref{tab:hyper}. The optional reference KL regularization uses the direction defined for each action space in Eq.~\eqref{eq:anchor}.

\begin{algorithm}[t]
\caption{MA-FPPO: decentralized execution}
\label{alg:inference}
\begin{algorithmic}[1]
\Require Online student $g_\theta$ and local encoder
\State Initialize each agent's local state
\For{environment steps $t=0,1,\ldots$ until termination}
 \For{each controlled agent $i$}
  \State Observe $o_t^i$ and update local input $h_t^i$
  \State Evaluate $g_\theta(h_t^i,0)$ once
  \State Choose $a_t^i=\arg\max_{b\in\mathcal A_i(h_t^i)}[g_\theta(h_t^i,0)]_b$ for discrete actions,
  \Statex \hspace{2\algorithmicindent}or $a_t^i=\clip(g_\theta(h_t^i,0),-1,1)$ for continuous actions
 \EndFor
 \State Execute joint action $\mathbf a_t$ and receive the next observations
\EndFor
\end{algorithmic}
\end{algorithm}
Local states are reset at every episode boundary. Algorithm~\ref{alg:inference} is the deterministic deployment protocol used for reported results. Stochastic evaluation uses the sampling rule in Algorithm~\ref{alg:ma-fppo} with parameters fixed.

\subsection{Analysis of Policy Gradients and Policy Changes}
\label{app:proofs}
\subsubsection{Optimizing Raw Actions under Clipping}
\begin{proposition}[Optimizing raw actions under clipping]
Let $a=f(u)$ be a deterministic measurable action map independent of the parameters. Assume $q_\theta(u\mid h)$ is positive and differentiable on a support independent of $\theta$, and differentiation can pass through the return expectation. Then the gradient of expected return can be expressed using $\nabla_\theta\log q_\theta(u\mid h)$. Write $f_\#q$ for the action distribution obtained by applying $f$ to samples from $q$, and $\KL$ for Kullback--Leibler divergence. For two raw policies $q,p$, $\KL(f_\#q\|f_\#p)\leq\KL(q\|p)$ whenever the KL between raw policies is finite.
\label{prop:clip}
\end{proposition}
\paragraph{Proof.}
Consider first a single decision with a fixed return functional $\mathcal R_{\mathrm{ret}}$ independent of $\theta$. Since the executed action is $f(u)$,
\begin{align*}
 \nabla_\theta\E_{u\sim q_\theta}[\mathcal R_{\mathrm{ret}}(f(u))]
 &=\int \mathcal R_{\mathrm{ret}}(f(u))\nabla_\theta q_\theta(u)\,du\\
 &=\E_{u\sim q_\theta}\left[\mathcal R_{\mathrm{ret}}(f(u))\nabla_\theta\log q_\theta(u)\right],
\end{align*}
under the usual conditions for differentiating under the integral. Applying the identity to finite trajectories and then taking the discounted limit gives the policy-gradient score at each raw sample under the stated regularity. The environment transition distribution depends on $f(u)$ and not directly on $\theta$. A state-dependent baseline can be subtracted without changing the exact score-function expectation. Approximate critics, GAE, minibatches, and PPO clipping introduce estimation error and change the optimization updates.

The KL inequality follows from the chain rule for relative entropy. Let $\mathsf P_1$ and $\mathsf P_2$ be the raw distributions and $a=f(u)$. Write $\mathsf P_1^a$ and $\mathsf P_2^a$ for the induced marginals of executed actions, and $\mathsf P_j(\cdot\mid a)$ for the conditional measure over raw actions for $j=1,2$. Then
\[
 \KL(\mathsf P_1\|\mathsf P_2)=\KL(\mathsf P_1^a\|\mathsf P_2^a)
 +\E_{a\sim \mathsf P_1^a}\left[\KL(\mathsf P_1(\cdot\mid a)\|\mathsf P_2(\cdot\mid a))\right]
 \geq\KL(\mathsf P_1^a\|\mathsf P_2^a).
\]
The induced action distribution accounts for non-invertible clipping and point masses at the action bounds. Evaluating a Gaussian density at the clipped sample would, in general, use a different random variable and would not reproduce the likelihood of the collected raw action. The proposition concerns likelihoods of raw actions and the induced distribution after clipping.

\subsubsection{Factorized Joint Policy Change}
\begin{proposition}[Factorized joint change]
For fixed local histories and product policies $\Pi=\prod_i\pi_i$ and $\widetilde\Pi=\prod_ip_i$ with each $\pi_i$ absolutely continuous with respect to $p_i$ and finite KL divergences,
$\KL(\Pi\|\widetilde\Pi)=\sum_i\KL(\pi_i\|p_i)$.
\label{prop:joint}
\end{proposition}
\paragraph{Proof.}
At fixed histories, write probabilities or densities with respect to common dominating measures. The product factorization gives
\begin{align*}
 \KL\!\left(\prod_i\pi_i\middle\|\prod_ip_i\right)
 &=\E_{\mathbf a\sim\prod_i\pi_i}\left[\sum_i\log\frac{\pi_i(a_i)}{p_i(a_i)}\right]\\
 &=\sum_i\E_{a_i\sim\pi_i}\left[\log\frac{\pi_i(a_i)}{p_i(a_i)}\right].
\end{align*}
The identity also holds with shared network parameters, because the factorization concerns conditional distributions. Agents' trajectories remain coupled through the environment. In discrete tasks the legal masks must be identical for collection and replay at the same history. Agents with a single legal action contribute zero KL. If the KL averaged over active agents is bounded by $\delta$, the corresponding sum is bounded by the number of active agents times $\delta$, with one contribution per active agent. The conversion applies at each fixed history. Empirical KL checks on collected samples do not enforce a uniform bound over histories.

\subsubsection{Team Policy Gradient}
\label{app:team-gradient}
The following analysis concerns the online \emph{sampling} policy. Collection parameters $\theta_{\mathrm{old}}$ change between rollouts, while $\theta_{\mathrm{ref}}$ remains fixed throughout online fine-tuning. Let $\chi$ be a Markov information state containing the joint histories and, when needed, the environment state and recurrent states. Using full histories avoids assuming that the implemented critic input alone is a sufficient statistic. Each $h_i$ is a projection of $\chi$ that is independent of the policy parameters during online fine-tuning. Write $\mu_{\theta,i}(w_i\mid h_i)$ for the categorical or raw Gaussian policy from Eq.~\eqref{eq:ratio}, and $M_\theta(\mathbf w\mid \chi)=\prod_i\mu_{\theta,i}(w_i\mid h_i)$. The executed action is a fixed function of $\mathbf w$. Define the normalized discounted occupancy
\begin{equation}
 d_\theta(\chi)=(1-\gamma)\sum_{t\geq0}\gamma^t\Pr_\theta(\chi_t=\chi),
 \label{eq:occupancy}
\end{equation}
where $\Pr_\theta$ is the state probability under $M_\theta$, with the analogous measure definition for continuous $\chi$. Assume $0<\gamma<1$, bounded rewards, a common initial distribution and transition mechanism, both independent of the policy, fixed legal supports at each history, and the regularity needed to interchange differentiation and expectation. Episodes are extended by absorbing terminal states with zero future reward. Let $A^\theta(\chi,\mathbf w)$ be the true team advantage, equal to the exact action value minus the exact state value.

\begin{proposition}[Team policy gradient with shared parameters]
Under the stated sampling and regularity assumptions,
\begin{equation}
 \nabla_\theta J(\theta)=\frac1{1-\gamma}
 \E_{\chi\sim d_\theta,\,\mathbf w\sim M_\theta}
 \left[A^\theta(\chi,\mathbf w)\sum_i\nabla_\theta\log\mu_{\theta,i}(w_i\mid h_i)\right].
 \label{eq:team-gradient}
\end{equation}
\label{prop:team-gradient}
\end{proposition}
\begin{proof}
For a finite trajectory, the policy contributes $\sum_t\sum_i\log\mu_{\theta,i}(w_t^i\mid h_t^i)$ to the log-likelihood. Differentiating expected discounted return yields the trajectory score multiplied by return. Rewards preceding a score have zero contribution by conditional expectation. Replacing each remaining future return by the conditional expected return gives the action value. Subtracting the value baseline, which is independent of the action, gives $A^\theta$. Taking the discounted infinite-horizon limit and expressing the sum using $d_\theta$ gives Eq.~\eqref{eq:team-gradient}. With shared parameters, the product rule sums all agents' scores.
\end{proof}
A common team advantage is thus compatible with decentralized sampling. The exact identity uses sampling from the discounted occupancy and true advantages. Uniform rollout averages change the state weighting, and finite-$\lambda$ GAE with a learned critic introduces advantage estimation error.

\subsubsection{Why Ratios for Individual Agents Give the Correct Local Direction}
\label{app:first-order}
Fix collection parameters $\theta_{\mathrm{old}}$ and abbreviate $A^{\mathrm{old}}=A^{\theta_{\mathrm{old}}}$. All expectations in the surrogate analysis use $d_{\theta_{\mathrm{old}}}$ and $M_{\theta_{\mathrm{old}}}$. The collection distributions and advantage are held fixed when differentiating. Define $\rho_i=\mu_{\theta,i}/\mu_{\theta_{\mathrm{old}},i}$ and centered surrogates
\begin{align}
 S_{\mathrm{agent}}(\theta)&=\E\left[A^{\mathrm{old}}\sum_i(\rho_i-1)\right],\nonumber\\
 S_{\mathrm{joint}}(\theta)&=\E\left[A^{\mathrm{old}}\left(\prod_i\rho_i-1\right)\right].
 \label{eq:surrogates}
\end{align}
Centering removes constants without changing gradients.
\begin{proposition}[First-order consistency at the collection policy]
At $\theta=\theta_{\mathrm{old}}$, both surrogates are zero and
\begin{equation}
 \left.\nabla S_{\mathrm{agent}}\right|_{\theta_{\mathrm{old}}}
 =\left.\nabla S_{\mathrm{joint}}\right|_{\theta_{\mathrm{old}}}
 =(1-\gamma)\nabla J(\theta_{\mathrm{old}}).
 \label{eq:first-order}
\end{equation}
For any $\epsilon>0$, replacing each $\rho_i A^{\mathrm{old}}$ by the corresponding term with PPO clipping preserves the gradient at $\theta_{\mathrm{old}}$, provided differentiation under the expectation is justified.
\label{prop:first-order}
\end{proposition}
\begin{proof}
Every ratio equals 1 at $\theta_{\mathrm{old}}$, and
\[
 \left.\nabla\prod_i\rho_i\right|_{\theta_{\mathrm{old}}}
 =\left.\sum_i\nabla\rho_i\right|_{\theta_{\mathrm{old}}}
 =\sum_i\nabla\log\mu_{\theta_{\mathrm{old}},i}(w_i\mid h_i).
\]
Substitution into Eq.~\eqref{eq:surrogates} and Proposition~\ref{prop:team-gradient} gives the result. The clipping interval contains 1 strictly inside the interval, so the clipped and unclipped sample objectives have the same derivative there \citep{ppo}.
\end{proof}
For a fixed collected minibatch, the same derivative identity holds with any fixed advantage estimates, including normalized GAE. Dividing by the actor's fixed positive $Z$ only rescales the minibatch gradient. Agents with only one legal action have zero score, so masking agents with one legal action leaves the gradient unchanged. An entirely inactive minibatch contributes zero policy gradient. Connecting the finite-sample identity to the exact gradient of $J$ additionally requires the sampling and advantage assumptions above. Unbiasedness additionally requires accounting for empirical centering, random minibatch normalization, and error in advantage estimation.

The agreement is local to $\theta_{\mathrm{old}}$: products and sums of ratios differ during subsequent optimization epochs. Reference KL and entropy gradients also change the update direction. Proposition~\ref{prop:first-order} therefore establishes local gradient agreement for the PPO component. The full update also depends on regularization and finite optimization steps.

\subsubsection{A Conditional Bound on Performance Improvement}
\label{app:performance-bound}
We give an argument about policy changes and discounted returns of the type underlying trust-region methods \citep{trpo}, specialized to the team's sampling policy. The bound holds under the stated sampling, advantage, and uniform KL assumptions. Consider old and new product sampling policies $M$ and $M'$ satisfying the preceding assumptions. Set
\begin{equation}
 A_* = \sup_{\chi,\mathbf w}|A^M(\chi,\mathbf w)|,
 \qquad \varepsilon_{\mathrm{TV}}=\sup_\chi D_{\mathrm{TV}}(M'(\cdot\mid \chi),M(\cdot\mid \chi)),
 \label{eq:drift-assumptions}
\end{equation}
where $A^M$ is the true advantage under $M$, $A_*$ bounds the advantage magnitude, and $D_{\mathrm{TV}}$ is total variation, one-half the $L_1$ distance. The quantity $\varepsilon_{\mathrm{TV}}$ is the largest policy change over information states. Bounded rewards ensure $A_*<\infty$. Define the surrogate gain for the joint policy $\mathcal G(M',M)=\E_{\chi\sim d_M,\,\mathbf w\sim M'}A^M(\chi,\mathbf w)$.

\begin{proposition}[Conditional team improvement]
The expected discounted returns obey
\begin{equation}
 J(M')-J(M)\geq\frac{\mathcal G(M',M)}{1-\gamma}
 -\frac{4\gamma A_*}{(1-\gamma)^2}\varepsilon_{\mathrm{TV}}^2.
 \label{eq:performance-tv}
\end{equation}
If $\sup_\chi\KL(M(\cdot\mid \chi)\|M'(\cdot\mid \chi))\leq\delta_{\mathrm{joint}}$, then
\begin{equation}
 J(M')-J(M)\geq\frac{\mathcal G(M',M)}{1-\gamma}
 -\frac{2\gamma A_*}{(1-\gamma)^2}\delta_{\mathrm{joint}}.
 \label{eq:performance-kl}
\end{equation}
\label{prop:performance}
\end{proposition}
\begin{proof}
Let $\zeta(\chi)=\E_{\mathbf w\sim M'}A^M(\chi,\mathbf w)$. Telescoping the old value along a trajectory of the new policy gives the performance-difference identity
\[
 J(M')-J(M)=\frac1{1-\gamma}\E_{\chi\sim d_{M'}}\zeta(\chi).
\]
Since $\E_{\mathbf w\sim M}A^M(\chi,\mathbf w)=0$, total variation bounds $\|\zeta\|_\infty\leq2A_*\varepsilon_{\mathrm{TV}}$. Let $P_M,P_{M'}$ denote induced transition kernels on $\chi$. Marginalizing actions through the common environment kernel cannot increase total variation, so $\sup_\chi\|P_{M'}(\cdot\mid \chi)-P_M(\cdot\mid \chi)\|_1\leq2\varepsilon_{\mathrm{TV}}$. The discounted occupancy equation gives
\[
 d_{M'}-d_M=\gamma(d_{M'}-d_M)P_{M'}+\gamma d_M(P_{M'}-P_M),
\]
which implies $\|d_{M'}-d_M\|_1\leq2\gamma \varepsilon_{\mathrm{TV}}/(1-\gamma)$. Replacing $d_{M'}$ by $d_M$ in the performance-difference identity incurs error at most
\[
 \frac{\|d_{M'}-d_M\|_1\|\zeta\|_\infty}{1-\gamma}
 \leq\frac{4\gamma A_*}{(1-\gamma)^2}\varepsilon_{\mathrm{TV}}^2.
\]
The resulting inequality proves Eq.~\eqref{eq:performance-tv}. Pinsker's inequality gives $\varepsilon_{\mathrm{TV}}^2\leq\delta_{\mathrm{joint}}/2$ and hence Eq.~\eqref{eq:performance-kl}.
\end{proof}
By Proposition~\ref{prop:joint}, uniform individual bounds $\sup_\chi\KL(\mu_i\|\mu_i')\leq\delta_i$ give the valid joint bound $\delta_{\mathrm{joint}}=\sum_i\delta_i$, where $\mu_i$ and $\mu_i'$ are the factors of $M$ and $M'$ at the same information state. A uniform average bound $\delta$ over $N$ agents therefore gives $N\delta$. For continuous actions the KL bounds apply directly to raw samples and the induced environment transitions. Proposition~\ref{prop:clip} also bounds KL for executed actions by KL for raw actions.

\paragraph{Relation to the implemented clipped surrogate.}
The gain $\mathcal G$ uses the joint likelihood ratio, whereas the implementation uses ratios for individual agents. With the exact advantage and expectation under the discounted occupancy and the old policy from Eq.~\eqref{eq:surrogates}, let
\begin{align}
 C_{\mathrm{clip}}&=\E\sum_i\min\{\rho_i A^{\mathrm{old}},\clip(\rho_i,1-\epsilon,1+\epsilon)A^{\mathrm{old}}\},\nonumber\\
 \Delta&=\E\left[A^{\mathrm{old}}\left\{\prod_i\rho_i-1-\sum_i(\rho_i-1)\right\}\right].
 \label{eq:surrogate-gap}
\end{align}
The quantity $C_{\mathrm{clip}}$ is the clipped surrogate summed across agents, and $\Delta$ is the residual between joint and individual expansions of likelihood ratios. Since $\E A^{\mathrm{old}}=0$, $C_{\mathrm{clip}}\leq S_{\mathrm{agent}}$ and $\mathcal G=S_{\mathrm{agent}}+\Delta$. Thus a sufficient condition for nonnegative improvement is
\begin{equation}
 C_{\mathrm{clip}}-|\Delta|
 \geq\frac{2\gamma A_*}{1-\gamma}\delta_{\mathrm{joint}}.
 \label{eq:sufficient-improvement}
\end{equation}
The decomposition uses an unnormalized agent sum and the true advantage. Eq.~\eqref{eq:ppo} uses normalized estimates from finite batches. The residual $\Delta$ has zero value and derivative at $\theta_{\mathrm{old}}$ under the stated regularity, but the local property does not bound the residual magnitude after several updates. Bounding the residual over finite updates requires additional control of the joint policy change.

\paragraph{Scope of the theoretical support.}
The derivations establish the multi-agent score identity, local surrogate consistency, and a conditional bound on policy improvement. Applying the bound requires true advantages and uniform control of the joint policy over histories. The implementation uses estimated advantages, minibatch updates, and KL stopping checks on collected samples. Reference KL regularization measures deviation from the initial policy, whereas Eq.~\eqref{eq:sufficient-improvement} compares consecutive policies. The theoretical bound therefore applies under the stated assumptions. The implemented optimizer and the contribution of flow initialization are evaluated empirically.

The bound concerns discounted return under the sampling policy. Reported deployment returns and win rates instead evaluate deterministic argmax actions or clipped Gaussian means.

\subsection{Effect of Latent Sampling on Stochastic Policy Gradients}
\label{app:latent-return}
We examine whether fixing $z=0$ can improve return by reducing errors introduced by online policy updates. The following model complements the discussion in Appendix~\ref{app:exploration-rationale} by connecting latent sampling to optimization error and deployment return. It isolates an effect of finite samples without assuming that zero selects a better initial behavior or requiring the random policy to preserve a nonzero latent effect.

\paragraph{Model and update assumptions.}
Consider a single decision with an unclipped scalar action, reward $r(a)=r_*-(a-a_*)^2/2$, and an affine student $g_{m,b}(z)=m+b^\top z$, where $a_*$ is the optimal action, $r_*$ is its reward, and $b\in\mathbb R^d$ for $d\geq1$. The two sampling policies are
\begin{equation}
 a_F=m_F+\sigma U,\qquad
 a_R=m_R+b^\top Z+\sigma U,
 \quad U\sim\Normal(0,1),\quad Z\sim\Normal(0,I_d),
 \label{eq:latent-model}
\end{equation}
where $F$ fixes $z=0$, $R$ samples the latent, $U$ and $Z$ are independent, and $\sigma>0$ is a common fixed exploration scale. Both policies start from the same arbitrary finite student parameters. The random policy updates both $m_R$ and $b$, with no restriction preventing $b=0$. Each update averages $n\geq1$ fresh independent policy gradient estimates from the conditional Gaussian likelihood and uses a constant step size $\alpha>0$. The baseline is the exact conditional value
\[
 V(Z)=\E_U[r(a_R)\mid Z]
 =r_*-\tfrac12\big[(m_R+b^\top Z-a_*)^2+\sigma^2\big],
\]
treated as constant when computing the gradient. The fixed policy uses the corresponding baseline at $z=0$. These updates use plain stochastic gradient ascent. Let $J_F(k)$ and $J_R(k)$ denote expected deployment returns after $k$ updates, averaging over training randomness and the deployment latent. Deployment removes $U$ from both policies and retains $Z$ only for $R$.

Define the constants
\[
 c_0=1+\frac2n,\qquad c_1=\frac3n,\qquad
 c_2=1+\frac{3d+5}{n},\qquad
 \nu^2=\frac52\sigma^2,\qquad q=\frac{\alpha\nu^2}{n}.
\]
\begin{proposition}[Return advantage in the affine model]
\label{prop:latent-return}
Under the preceding assumptions, $J_F(k)>J_R(k)$ for every $k\geq1$. If $0<\alpha<1/(c_2+dc_1)$, both expected returns have finite limits and
\begin{equation}
 \lim_{k\to\infty}\big[J_F(k)-J_R(k)\big]
 =\frac{dq\,[2-(c_0-c_1)\alpha]^2}
 {2(2-c_0\alpha)\mathcal D_\alpha}>0,
 \qquad
 \mathcal D_\alpha=(2-c_0\alpha)(2-c_2\alpha)-dc_1^2\alpha^2.
 \label{eq:latent-return-gap}
\end{equation}
\end{proposition}

\begin{proof}
Write $e=m_R-a_*$, $x=(1,Z^\top)^\top$, and $w=(e,b^\top)^\top$. The conditional Gaussian score is $Ux/\sigma$. Substituting the reward and baseline gives the single-sample gradient
\begin{equation}
 \widehat\nabla J=-\big[U^2(w^\top x)+H\big]x,
 \qquad H=\tfrac\sigma2(U^3-U).
 \label{eq:latent-score}
\end{equation}
Gaussian moments give $\E U^2=1$, $\E U^4=3$, $\E H=\E[U^2H]=0$, and $\E H^2=\nu^2$. Hence $\E[\widehat\nabla J\mid w]=-w$, the exact gradient of $r_*-(\|w\|^2+\sigma^2)/2$.

Let $A_k^F=\E(m_{F,k}-a_*)^2$, $A_k^R=\E(m_{R,k}-a_*)^2$, and $B_k=\E\|b_k\|^2$. For one sample from the random policy, the expected squared gradient for $m_R$ is $3(e^2+\|b\|^2)+\nu^2$, while the expected squared norm of the gradient for $b$ is $3de^2+3(d+2)\|b\|^2+d\nu^2$. Averaging $n$ independent samples divides the gradient covariance by $n$. Expanding the squared parameter updates therefore yields
\begin{align}
 A_{k+1}^F&=(1-2\alpha+c_0\alpha^2)A_k^F+\alpha q,\nonumber\\
 \begin{pmatrix}A_{k+1}^R\\B_{k+1}\end{pmatrix}
 &=\underbrace{\begin{pmatrix}
 1-2\alpha+c_0\alpha^2&c_1\alpha^2\\
 dc_1\alpha^2&1-2\alpha+c_2\alpha^2
 \end{pmatrix}}_{T}
 \begin{pmatrix}A_k^R\\B_k\end{pmatrix}
 +\alpha q\binom1d.
 \label{eq:latent-moments}
\end{align}
The Gaussian identity $\E[(b^\top Z)^2\|Z\|^2]=(d+2)\|b\|^2$ supplies the last coefficient.

For $D_k=A_k^R-A_k^F$, the common initialization gives $D_0=0$ and
\[
 D_{k+1}=(1-2\alpha+c_0\alpha^2)D_k+c_1\alpha^2B_k.
\]
All coefficients are nonnegative, so $D_k\geq0$. The positive source $d\alpha q$ gives $B_k>0$ for $k\geq1$, even if $b_0=0$. The expected deployment returns are $J_F(k)=r_*-A_k^F/2$ and $J_R(k)=r_*-(A_k^R+B_k)/2$, proving the strict finite-update gap $(D_k+B_k)/2$.

Under the stated step size, $T$ is nonnegative and each row sum is less than one. Its spectral radius is therefore below one, and the moment recursions converge. Solving their limiting equations gives
\[
 A_\infty^F=\frac q{2-c_0\alpha},\qquad
 B_\infty=\frac{dq[2-(c_0-c_1)\alpha]}{\mathcal D_\alpha},\qquad
 A_\infty^R-A_\infty^F=
 \frac{c_1\alpha}{2-c_0\alpha}B_\infty.
\]
The same condition makes $\mathcal D_\alpha>0$. Substitution into the return difference proves Eq.~\eqref{eq:latent-return-gap}.
\end{proof}

\paragraph{Conclusion and relation to online fine-tuning.}
The model shows that finite samples continually perturb the learned dependence on the latent, even when its optimal coefficient is zero. This error also increases the error in the action center through Eq.~\eqref{eq:latent-moments}. Fixing $z=0$ removes that update channel while retaining Gaussian exploration and updates to the center. The result concerns limits of expected returns under a constant step size, rather than convergence of individual parameter trajectories. With exact gradients, both policies can reach the same optimum.

The proposition provides a conditional optimization mechanism consistent with the comparison in Figure~\ref{fig:direct-flow-policy-finetuning}. Its affine policy and quadratic reward isolate the effect of sampling. The exact baseline also conditions on the latent, whereas the implemented centralized critic uses the environment state. Nonlinear policies, estimated critics, learned exploration scales, PPO updates, and changing state distributions fall outside the analysis. The empirical comparison supports the complete policy construction, while the relative contribution of this mechanism remains unestablished.

\end{document}